\documentclass[12pt]{article}
\usepackage[colorlinks=true,linkcolor=blue,citecolor=Red]{hyperref}
\usepackage{breakcites}

\usepackage{mathtools}

\usepackage{subfigure}

\usepackage{algorithm}
\usepackage{algpseudocode}

\usepackage{bbm}

\usepackage{latexsym}
\usepackage{graphics}
\usepackage{amsmath}
\usepackage[algo2e]{algorithm2e}
\usepackage{xspace}
\usepackage{amssymb}
\usepackage{psfrag}
\usepackage{epsfig}
\usepackage{amsthm}
\usepackage{url}
\usepackage{pst-all}
\usepackage{bbm}

\usepackage{color}

\definecolor{Red}{rgb}{1,0,0}
\definecolor{Blue}{rgb}{0,0,1}
\definecolor{Olive}{rgb}{0.41,0.55,0.13}
\definecolor{Yarok}{rgb}{0,0.5,0}
\definecolor{Green}{rgb}{0,1,0}
\definecolor{MGreen}{rgb}{0,0.8,0}
\definecolor{DGreen}{rgb}{0,0.55,0}
\definecolor{Yellow}{rgb}{1,1,0}
\definecolor{Cyan}{rgb}{0,1,1}
\definecolor{Magenta}{rgb}{1,0,1}
\definecolor{Orange}{rgb}{1,.5,0}
\definecolor{Violet}{rgb}{.5,0,.5}
\definecolor{Purple}{rgb}{.75,0,.25}
\definecolor{Brown}{rgb}{.75,.5,.25}
\definecolor{Grey}{rgb}{.5,.5,.5}

\newcommand{\ind}{\mathbbm{1}}

\newcommand{\R}{\mathbb{R}}

\newcommand{\ReLU}{\texttt{ReLU}}
\newcommand{\SReLU}{\texttt{S-ReLU}}
\newcommand{\Step}{\texttt{Step}}

\newcommand{\EmpRisk}[1]{\widehat{\mathcal{L}}\left(#1\right)}

\newcommand{\NN}{\texttt{NN} }
\newcommand{\SGM}{\texttt{SGM}}
\newcommand{\Er}{{\rm er}}
\newcommand{\HatEr}[3]{\widehat{{\rm er}}_{#3}^{#1}\left(#2\right)}

\renewcommand{\R}{\mathbb{R}}

\newcommand{\distr}{\stackrel{d}{=}}

\newcommand{\ignore}[1]{\relax}

\newlength\myindent
\newtheorem{theorem}{Theorem}[section]

\newtheorem{lemma}[theorem]{Lemma}

\newtheorem{proposition}[theorem]{Proposition}
\newtheorem{coro}[theorem]{Corollary}

\newtheorem{definition}[theorem]{Definition}
\newtheorem{Assumption}[theorem]{Assumption}


\newcounter{parentnumber}

\makeatletter
\def\BState{\State\hskip-\ALG@thistlm}
\makeatother

\definecolor{Red}{rgb}{1,0,0}
\definecolor{Blue}{rgb}{0,0,1}
\definecolor{Olive}{rgb}{0.41,0.55,0.13}
\definecolor{Green}{rgb}{0,1,0}
\definecolor{MGreen}{rgb}{0,0.8,0}
\definecolor{DGreen}{rgb}{0,0.55,0}
\definecolor{Yellow}{rgb}{1,1,0}
\definecolor{Cyan}{rgb}{0,1,1}
\definecolor{Magenta}{rgb}{1,0,1}
\definecolor{Orange}{rgb}{1,.5,0}
\definecolor{Violet}{rgb}{.5,0,.5}
\definecolor{Purple}{rgb}{.75,0,.25}
\definecolor{Brown}{rgb}{.75,.5,.25}
\definecolor{Grey}{rgb}{.5,.5,.5}
\definecolor{Pink}{rgb}{1,0,1}
\definecolor{DBrown}{rgb}{.5,.34,.16}
\definecolor{Black}{rgb}{0,0,0}

\usepackage{float}

\usepackage{float}

\author{{\sf David Gamarnik}\thanks{MIT; e-mail: {\tt gamarnik@mit.edu}. Research supported  by the NSF grants DMS-2015517.}
\and
{\sf Eren C. K{\i}z{\i}lda\u{g}}\thanks{MIT; e-mail: {\tt kizildag@mit.edu}.}
\and 
{\sf Ilias Zadik}\thanks{NYU; e-mail: {\tt zadik@nyu.edu}. Research supported by a CDS Moore-Sloan Postdoctoral Fellowship.}
}

\begin{document}

\title{Self-Regularity of Non-Negative Output Weights \\ for Overparameterized Two-Layer Neural Networks}
\date{\today}

\maketitle
\begin{abstract}
We consider the problem of finding a two-layer neural network with sigmoid, rectified linear unit (ReLU), or binary step activation functions that ``fits" a training data set as accurately as possible as quantified by the training error; and study the following question: \emph{does a low training error guarantee that the norm of the output layer (outer norm) itself is small?} We answer affirmatively this question for the case of non-negative output weights. Using a simple covering number argument, we establish that under quite mild distributional assumptions on the input/label pairs; any such network achieving a small training error on polynomially many data necessarily has a well-controlled outer norm. Notably, our results (a) have a polynomial (in $d$) sample complexity, (b) are independent of the number of hidden units (which can potentially be very high), (c) are oblivious to the training algorithm; and (d) require quite mild assumptions on the data (in particular  the input vector $X\in\mathbb{R}^d$ need not have independent coordinates). We then leverage our bounds to establish generalization guarantees for such networks through \emph{fat-shattering dimension}, a scale-sensitive measure of the complexity class that the network architectures we investigate belong to. Notably, our generalization bounds also have good sample complexity (polynomials in $d$ with a low degree), and are in fact near-linear for some important cases of interest.
\end{abstract}
\newpage
\tableofcontents
\newpage

\section{Introduction}
Neural network (\texttt{NN}) architectures achieved a great deal of success in practice. An ever-growing list of their applications includes image recognition~\cite{he2016deep}, image classification~\cite{krizhevsky2012imagenet},  speech recognition~\cite{mohamed2011acoustic}, natural language processing~\cite{collobert2008unified}, game playing~\cite{silver2017mastering} 
and more. Despite this great empirical success, however, a rigorous understanding of these networks is still an ongoing quest. 

A common paradigm in classical statistics is that  \emph{overparameterized} models, that is, models with more parameters than necessary, pick on the idiosyncrasies of the training data itself---dubbed as \emph{overfitting}; and as a consequence, tend to  predict \emph{poorly} on the unseen data---called poor \emph{generalization}. The aforementioned success of the \NN architectures, however, stands in the face of this conventional wisdom; and a growing body of recent literature, starting from~\cite{zhang2016understanding}, has demonstrated exactly the opposite effect for a broad class of \NN models: 
even though the number of parameters, such as the number
of hidden units (neurons), of a \NN significantly exceeds the sample size, and a perfect (zero) \emph{in-training error} is achieved (commonly called as
\emph{data interpolation}); they still retain a good generalization ability. Some partial and certainly very incomplete list of references to this
point are found in~\cite{du2018gradient,li2018learning,gunasekar2018implicit,goldt2019dynamics,belkin2019reconciling,arora2019exact}. Defying  statistical intuition even further, it was established empirically in~\cite{belkin2019reconciling}
that beyond a certain point, increasing the number of parameters increases out of sample accuracy.

Explaining this conundrum is arguably one of the most vexing current problems in the field of theoretical machine learning. Standard Vapnik-Chervonenkis (VC) theory do not help explaining the good generalization ability of overparameterized \NN models, since the VC-dimension of these networks grows (at least) linearly in the number of parameters~\cite{harvey2017nearly,bartlett2019nearly}. These findings fueled significant research efforts aiming at understanding the generalization ability of such networks. One such line of research is the algorithm-independent front; and is through the lens of controlling the norm of the matrices carrying weights~\cite{neyshabur2015norm,bartlett2017spectrally,liang2017fisher,golowich2017size,dziugaite2017computing}, PAC-Bayes theory~\cite{neyshabur2017pac,neyshabur2017exploring}, and compression-based bounds~\cite{arora2018stronger}, among others. A major drawback of these approaches, however, is that they require certain norm constraints on the weights considered; therefore making their guarantees \emph{a posteriori} in nature: whether or not the weights of the \NN are bounded (hence a good generalization holds) can be determined only after the training process is complete. An alternative line of research (detailed below) focuses on the end results of the algorithms, and potentially yields \emph{a priori} guarantees: for instance, relatively recently, Arora et al.\,gave in~\cite{arora2019fine} \emph{a priori} guarantees for the solution found by the \emph{gradient descent} algorithm under random initialization. 
%The generalization ability of such networks has been studied, among others,  These approaches con %Major drawbacks of these approaches include the following: (a) they require certain constraints on the weights; and more importantly (b) these guarantees are mostly . That is, . One exception

A predominant explanation of the aforementioned phenomenon (that the overparameterization does not hurt the generalization ability of the \NN architectures) which has emerged recently is based on the idea of \emph{self-regularization}. Specifically, it is argued that even though there is an abundance of parameter choices perfectly fitting \emph{(interpolating)} the data (and thus achieving zero in-training error); the algorithms used in training the models, such as the \emph{gradient descent}  and its many variants such as \emph{stochastic gradient descent}, \emph{mirror descent}, etc., tend to find solutions
 which are regularized according to some additional criteria, such as small norms, thus introducing algorithm dependent \emph{inductive bias}. Namely, the algorithms implemented for minimizing training error ``prefer" certain kinds of solutions. The use of these  solutions for model building in particular is believed to result in low generalization errors. Thus a significant research effort (as was partially mentioned above) was devoted to the analysis of the end results of the implementation of such algorithms. This line of research include the analysis of the end results of the gradient descent~\cite{brutzkus2017globally,frei2019algorithm}, stochastic gradient descent~\cite{hardt2016train,brutzkus2017sgd,li2018learning,cao2019generalization}, as well as the stochastic gradient Langevin dynamics~\cite{mou2018generalization}.
 
In this paper, we consider two-layer \NN models \eqref{eq:NN-fnc}---also known as \emph{shallow} architectures---consisting of an arbitrary number $\overline{m}\in\mathbb{N}$ of hidden units and sigmoid, rectified linear unit (ReLU), or binary step activations---activations that are arguably among the most popular practical choices---and investigate the following question: \emph{to what extent a low training error itself places a 
restriction on the weights of the learned \texttt{NN}?} We take an algorithm-independent route; and establish the following ``picture", under the assumption that the output weights $a=\left(a_i:1\le i\le \overline{m}\right)\in\R^{\overline{m}}$ of the ``learned" \NN are \emph{non-negative}. When the number $N$ of training samples is at least an explicit (low-degree) polynomial function in $d$, $N=d^{O(1)}$, the norm $\|a\|_1$ of the output weights $a\in\R_{\ge 0}^{\overline{m}}$ of {\bf any} \NN model achieving a small training error is well-controlled: $\|a\|_1=O(1)$, with high probability over the training data set. In particular, for the ReLU and step networks, we obtain a near-linear sample complexity bound, $N=\Theta(d \log d)$ for such a result to hold. Note that a condition such as the non-negativity of $a_i$ is necessary in a strict sense for such a bound on $\|a\|_1$. Indeed, notice that by  growing the width $\overline{m}$ arbitrarily and appropriately choosing alternating signs for the new weights $a_i$; one can introduce cancellations and make $\|a\|_1$ to explode; while keeping the training error unchanged. %new weights $a_i$ with , one can introduce many cancellations%Indeed, notice that by growing arbitrarily the width and appropriately choosing alternating signs of the new weights ai, one can introduce a cancellation which makes the outer norma of a explode while it keeps the training error unchanged
%Note that the non-negativity assumption is necessary in some sense as without it $\|a\|_1$ can be made arbitrarily large by playing with the signs of $a_i$ and introducing many cancellations. 

Our results are established using elementary tools, in particular through an $\epsilon$-net argument (Definition~\ref{def:eps-net}). Notably, our results (a) are independent of the number $\overline{m}$ of the hidden units (which can potentially be quite large), (b) are oblivious to the way the training is done (that is, independent of the choice of the training algorithm); and (c) are valid under quite mild distributional assumptions on the input/label pairs $(X,Y)\in\R^d\times \R$. In particular, the coordinates of $X$ need not be independent. 

%%%%%% OPERATE HERE %%%%%%%%%%%%
Moreover, a bounded outer norm for such network models implies a well-controlled \emph{fat-shattering dimension} (FSD)~\cite{bartlett1998sample}---a measure of the complexity of the model class achieving a low training error. In Section~\ref{sec:genel}, we leverage our outer norm bounds and the FSD to establish generalization guarantees for the networks that we investigate. %A bounded outer norm implies well-controlled fat-shattering dimension 
%Moreover, a bounded outer norm implies ``low-complexity" for such \NN models with sigmoidal activation, see~\cite{bartlett1998sample} and Section~\ref{sec:genel}. Thus, our results have implications about the generalization abilities of such networks. 
%%%%%% OPERATE HERE %%%%%%%%%%%%
%Finally, we provide an informal argument which shows that, strictly speaking, the non-negativity assumption on the output weights is necessary: by playing with the signs and introducing many cancellations, one may generate networks with small training error and arbitrarily large output norm. 
%A preliminary version of our results appeared in the preprint~\cite{emschwiller2020neural}. 
The current paper presents significantly strengthened versions and extensions of some results appeared in our preprint~\cite{emschwiller2020neural}. %obtained by a much simpler machinery. 
%For brevity, we compress/omit certain proofs. We plan to give full proofs and certain extensions in a longer version of this paper. 

% We provide compressed proofs for brevity; and plan to provide full proofs/extensions in a forthcoming longer version.
%Another line of research, which is more relevant to us, assumes the existence of a planted ground truth model generating the labels. Examples along this line include the works \cite{du2018many,ma2018priori,imaizumi2018deep}. Our approach in Section \ref{sec:data-no} are similar to this regard in the sense that it depends only on the data, and moreover, stated as  a geometric condition; yet, suffer from one potential caveat that, the hardness of the training task is not our main focus, and . This has been addressed recently in the work by Arora et al. \cite{arora2019fine}, which gave an {\em a priori} guarantee, for the solution found by the gradient descent algorithm. 

%Furthermore, success of this great practical success 
%Success of modern learning methods such as neural networks 
%rior attempts to explain this conundrum were through the lens of the expressive ability of these networks, going as early as Barron~\cite{barron1994approximation}. Recent work along this line focused on deep and sparse networks~\cite{telgarsky2016benefits,eldan2016power,poggio2017and,bolcskei2019optimal}; and the expressive abilities of such networks are relatively well-understood. 
\subsection*{Preliminaries}
We commence this section with a list of notational convention that we follow throughout. 
\paragraph{Notation.} %We commence this section with a list of notational convention that we follow throughout. 
The set of reals, non-negative reals, and positive integers are denoted respectively by $\mathbb{R},\mathbb{R}_{\ge 0}$, and $\mathbb{N}$. For any set $S$, $|S|$ denotes its cardinality. For any $N\in\mathbb{N}$, $[N]\triangleq \{1,2,\dots,N\}$. For any $v\in\R^n$, its $\ell_p$ norm  %, $(\sum_{1\le i\le n}|v_i|^p)^{1/p}$, 
is denoted by $\|v\|_p$. For $u,v\in \R^n$, their Euclidean inner product %, $\sum_{1\le i\le n}u_i v_i$,
is denoted by $u^Tv$. For any $r\in\mathbb{R}$, $\exp(r)$ denotes $e^r$; and $\ln(r)$ denotes the logarithm of $r$ base $e$. For any ``event" $E$; $\ind\{E\}=1$ when $E$ is true; and $\ind\{E\}=0$ when $E$ is false. $\SGM(x)$ denotes the sigmoid activation function, $1/(1+\exp(-x))$; $\ReLU(x)$ denotes the ReLU activation function, $\max\{x,0\}$; and $\Step(x)$ denotes the (binary) step activation, $\ind\{x\ge 0\}$. $X\distr\mathcal{N}(0,\Sigma)$ if $X$ is a zero-mean multivariate normal vector with covariance $\Sigma$.  A random variable $U$ is symmetric around zero if $U$ and $-U$ have the same distribution, that is $U\distr -U$. For any random variable $U$, (if finite) its moment generating function (MGF) at $s\in\mathbb{R}$, $\mathbb{E}[\exp(sU)]$, is denoted  by $M_U(s)$. Finally, $\Theta(\cdot),o(\cdot), O\left(\cdot\right)$ are the standard asymptotic order notations. %for comparing the growth of two real-valued sequences.

\paragraph{ Setup.} A two-layer \NN $(a,W)\in\R^{\overline{m}}\times \R^{\overline{m}\times d}$ with $\overline{m}$ hidden units (neurons) computes, for each $X\in\R^d$, \begin{equation}\label{eq:NN-fnc}
\sum_{1\le j\le \overline{m}}a_j\sigma\left(w_j^T X\right).
\end{equation}Here, $\sigma(\cdot)$ is the activation; $w_j\in\R^d$, the $j^{\rm th}$ row of $W$, carries the weights of neuron $j$; and $a=\left(a_j:1\le j\le \overline{m}\right)\in\R^{\overline{m}}$ carries the output weights. $\|a\|_1$ is referred to as the \emph{outer norm}. We assume $a_j\ge 0$ for $j\in[\overline{m}]$. This non-negativity assumption appears often in the theoretical study of this model: see~\cite{ge2017learning,diakonikolas2020algorithms,li2020learning} for generic $a\in\R^{\overline{m}}_{\ge 0}$; and \cite{du2018power,safran2018spurious,zhang2019learning,goel2018learning} for the case $a_j$ are equal to the same positive number. 

Our study of \NN models under the non-negativity assumption is also partly motivated from an applied point of view, in that, non-negativity is inherent to many data sets appearing in practice, including audio data and data on muscular activity~\cite{smaragdis2017neural,WikiNMF} and allow interpretability. Furthermore, non-negativity is also a commonly used assumption in the context of matrix factorization, termed as the \emph{non-negative matrix factorization problem (NMF)}: given a matrix $M\in\R^{n\times m}$ with non-negative entries and an integer $r\ge 1$, the goal of the NMF is to find matrices $A\in\R^{n\times r}$ and $W\in\R^{r\times m}$ with non-negative entries such that the product $AW$ is as ``close" to $M$ as possible; as quantified, e.g., by the Frobenius norm. This problem is a fundamental problem appearing in many practical applications, including information retrieval, document clustering, image segmentation, demography and chemometrics, see~\cite{arora2016computing} and the references therein. Moreover, NMF is also related to the neural network models that we consider herein with a non-negative activation $\sigma(\cdot)$: observe that in the context of \NN models we consider, given data $(X_i,Y_i)$, $1\le i\le N$, the goal of the learner is to find a $(a,W)\in\R^{\overline{m}}\times \R^{\overline{m}\times d}$ such that $Y_i$ and $a^T \sigma\left(WX_i\right)$ are as close as possible, as quantified by the $\ell_2$ norm (here, $\sigma$ acts coordinate-wise to the vector $WX$). See also~\cite[Section~6]{gabrie2019towards} for a more rigorous connection between shallow \NN models, matrix factorization and message passing algorithms. In addition to its key role in the NMF problem; the non-negativity was also argued as a natural assumption for representing objects in the seminal papers by Lee and Seung~\cite{lee1999learning,leeseungnips}; and also has roots in biology, in particular in the context of neuronal firing rates, see~\cite{hoyer2002non}, and the references therein.

In the sequel, $d\in\mathbb{N}$ is reserved for the input dimension; and $\overline{m}\in\mathbb{N}$ is reserved for the number of neurons. We consider herein two-layer \NN models with sigmoid, $\SGM(x)$; rectified linear unit, $\ReLU(x)$; and binary step, $\Step(x)$, activation functions. We refer to these as sigmoid, ReLU; and step networks, respectively. The sigmoid and the ReLU are arguably among the most popular practical choices. The step function, on the other hand, is one of the initial activations considered in the \NN literature, and is inspired from a biological point of view: it resembles the firing pattern of a neuron, an initial motivation for studying \NN architectures.

Given the data $(X_i,Y_i)\in\R^d\times \R$, $1\le i\le N$, consider the problem of finding a two-layer \NN $(a,W)\in\R^{\overline{m}}\times \R^{\overline{m}\times d}$ which ``fits" the data as accurately as possible. This is achieved by solving the so-called \emph{empirical risk minimization} problem, where the accuracy is quantified by the \emph{training error}
\begin{equation}\label{eq:training-error}
 \EmpRisk{a,W}\triangleq \frac1N \sum_{1\le i\le N}\left(Y_i - \sum_{1\le j\le \overline{m}}a_j\sigma\left(w_j^T X_i\right)\right)^2.
\end{equation}
One then runs a training algorithm, e.g., the gradient descent algorithm or one of its variants (such as stochastic gradient descent or mirror descent), to find an $(a,W)$ with a small $\EmpRisk{a,W}$.
%The ERM problem, $\min_{(a,W)}\EmpRisk{a,W}$, can be solved,  by using the gradient descent algorithm or , . Our focus is on the approximate minima $(a,W)$ of this problem. 

\paragraph{ Distributional assumption.} We study the case where the input/label pairs $(X_i,Y_i)$, $1\le i\le N$, are i.i.d. samples of a distribution on $\R^d\times \R$ (which is potentially unknown to the learner). For our outer norm bounds, we assume that their distribution satisfies the following.
\begin{itemize}
    \item We assume the input $X\in\R^d$ satisfies $\mathbb{P}\left(\|X\|_2^2\le Cd\right)\ge 1- \exp\left(-\Theta(d)\right)$ for some constant $C>0$.
    \item We assume the label $Y$ is such that $\mathbb{E}[|Y|]\triangleq M<\infty$. 
\end{itemize}
Later in Section~\ref{sec:genel} when we study generalization guarantees, we consider a stronger assumption on labels: we assume the labels $Y$ are bounded, that is, for some $M>0$, $|Y|\le M$ almost surely. 
%The input $X\in\R^d$ satisfies $\mathbb{P}\left(\|X\|_2^2\le Cd\right)\ge 1- \exp\left(-\Theta(d)\right)$ for some constant $C>0$. For the label $Y$, $\mathbb{E}[|Y|]\triangleq M<\infty$. %we assume there exists a $C>0$ such that $\mathbb{P}\left(\|X_i\|_2^2\le Cd\right)\ge 1-\exp\left(-\Theta(d)\right)$. %For the input $X\in\R^d$, we assume there exists a $C>0$ such that $\mathbb{P}\left(\|X_i\|_2^2\le Cd\right)\ge 1-\exp\left(-\Theta(d)\right)$.
%For the label $Y$,  
%with bounded labels: there is an $M>0$ such that $|Y_i|\le M$ almost surely. %We keep $M$ throughout. % Later in Section~\ref{sec:generalization} when we study the generalization performance for pattern classification problems, we focus on the specific case $Y_i\in\{-1,1\}$ (thus $M=1$). 

 These assumptions are quite mild. For instance, $X\in\R^d$ need not have i.i.d.\,coordinates. Moreover, most real data sets indeed have bounded labels~\cite{du2018gradient}; and this bounded label assumption is employed extensively in literature, see e.g.~\cite{ge2019mildly,arora2019fine,du2019gradient,goel2019learning,li2020convolutional}. Our next assumption regards the number $N$ of samples.
 %and the bounded label assumption holds for most real data sets~\cite{du2018gradient}. Furthermore, the bounded label assumption has been adopted 
 \begin{Assumption}\label{sump:N-exp-d}
 Throughout, we assume  that the sample size $N$ satisfies $N\le \exp(cd)$ for some $c>0$.
  \end{Assumption}
  Assumption~\ref{sump:N-exp-d} is required for technical reasons: observe that since $\mathbb{P}\left(\|X_i\|_2^2 >Cd\right)\le \exp(-\Theta(d))$, it holds, by a union bound, that
  \[
  \mathbb{P}\Bigl(\|X_i\|_2^2 \le Cd,1\le i\le N\Bigr)\ge 1-N\exp(-\Theta(d)).
  \]
  For this bound to be non-vacuous, $N$ should at most be $\exp(cd)$ for a small enough $c>0$. This assumption, again, is very benign due to obvious practical reasons. Moreover, it suffices to have $N\ge {\rm poly}(d)$ for our results to hold. 
  
\paragraph{ Nets and Covering Numbers.} The crux of our proofs is the so-called \emph{$\epsilon-$net argument}~\cite{vershynin2010introduction,vershynin2018high}. This (rather elementary) argument is also known as the \emph{covering number argument}; and has been employed extensively in the literature; including compressed sensing, machine learning and probability theory.
\begin{definition}\label{def:eps-net}%[{\bf Nets and Covering Numbers}]
Let $\epsilon>0$. Given a metric space $(X,\rho)$, a subset $\mathcal{N}_\epsilon\subset X$ is called an \emph{$\epsilon-$net} of $X$ if, for every $x\in X$, there is a $y\in\mathcal{N}_\epsilon$ such that $\rho(x,y)\le \epsilon$. The smallest cardinality of such an $\mathcal{N}_\epsilon$, if finite, is called the \emph{covering number} of $X$, denoted by $\mathcal{N}(X,\epsilon)$. 
\end{definition}
The next result, verbatim from~\cite[Corollary~4.2.13]{vershynin2018high}, is an upper bound on the covering number of the Euclidean ball.
\begin{theorem}\label{lemma:card-net} Let $B_2(0,R)\triangleq \left\{x\in\R^d:\|x\|_2\le R\right\}$. Then for $R\ge 1$ and any $\epsilon>0$
\[\mathcal{N}\left(B_2(0,R),\epsilon\right)\le (3R/\epsilon)^d.
\]
\end{theorem} 
\paragraph{ Paper organization.} The rest of the paper is organized as follows. Our main results on the self-regularity of output weights are presented in Section~\ref{sec:main-bd}. In particular, see Sections~\ref{sec:main-bd-sgm},~\ref{sec:main-bd-ReLU}, and~\ref{sec:main-bd-Step} for the cases of sigmoid, ReLU, and step networks, respectively. By leveraging our outer norm bounds and employing earlier results on the fat shattering dimension, we establish in Section~\ref{sec:genel} generalization guarantees. We outline several future directions in Section~\ref{sec:conclude}. Finally, we present our proofs in Section~\ref{sec:pfs}.
\section{Outer Norm Bounds}\label{sec:main-bd}
In this section, we establish the self-regularity of the output weights for the aforementioned networks. That is, we establish that the outer norms of sigmoid, ReLU, and step networks with non-negative output weights achieving a small training error~\eqref{eq:training-error} on polynomially many data is $O(1)$. 
\subsection{Self-Regularity for the Sigmoid Networks}\label{sec:main-bd-sgm}
Our first focus in on the sigmoid networks. This object, for each $X\in\R^d$, computes the function~\eqref{eq:NN-fnc} with $\sigma=\SGM(\cdot)=(1+\exp(-x))^{-1}$. Our first main result establishes an outer norm bound for this architecture. 
\begin{theorem}\label{thm:sigmoid}
Let $\delta,M,R>0$; and $\left(X_i,Y_i\right)\in\R^d\times \R$, $i\in[N]$ be i.i.d. data with $\mathbb{E}\left[\left|Y_i\right|\right]=M<\infty$; where $N$ satisfies Assumption~\ref{sump:N-exp-d}. For any $\overline{m}\in\mathbb{N}$, define 
\[
\mathcal{S}\left(\overline{m},\delta,R\right)=\left\{\left(a,W\right)\in \R_{\ge 0}^{\overline{m}}\times \R^{\overline{m}\times d}:\max_{1\le j\le \overline{m}}\|w_j\|_2\le R,\,\,\EmpRisk{a,W}\le \delta^2\right\},
\]
where $\EmpRisk{\cdot}$ is defined in \eqref{eq:training-error} with $\sigma(\cdot)=\SGM(\cdot)$. Suppose, in addition, that the random variable $w^TX\in\R$ is symmetric around zero for every $w\in\R^d$. Then, 
\begin{equation}\label{eq:sgm-pb-bd}
\mathbb{P}\left(\sup_{(a,W)\in \mathcal{S}\left(\delta,R\right)}\|a\|_1 \le 3(1+e)(\delta+2M)\right)\ge 1-\left(3R\sqrt{Cd}\right)^d\exp\left(-\Theta(N)\right)-N\exp\left(-\Theta(d)\right)-o_N(1),
\end{equation}
where $
\mathcal{S}(\delta,R)\triangleq \bigcup_{\overline{m}\in\mathbb{N}}\mathcal{S}\left(\overline{m},\delta,R\right)$.
\end{theorem}
\begin{coro}\label{coro:sigmoid}
Let $R=\exp\left(d^{O(1)}\right)$. Then, under the assumptions of Theorem~\ref{thm:sigmoid}; it holds w.h.p.\,that $\sup_{(a,W)\in\mathcal{S}(\delta,R)}\|a\|_1 \le 3(1+e)(\delta+2M)$, provided $N\ge d^{O(1)}$.
\end{coro}
The proof of Theorem~\ref{thm:sigmoid} is provided in Section~\ref{sec:pf-sigmoid}.

Above, $o_N(1)$ is a function which depends only on the distribution of $Y$ and $N$; and tends to zero as $N\to\infty$.
Several remarks are now in order. Theorem~\ref{thm:sigmoid} states that \emph{any} two-layer sigmoid \NN which (a) consists of internal weights $w_j$ bounded in norm by an exponentially large (in $d$) quantity and non-negative output weights; and (b) achieves a small training error on a \emph{sufficiently large} data set, has a \emph{well-controlled} outer norm. It is worth noting that Theorem~\ref{thm:sigmoid} is oblivious to  how the training is done: this result not only applies to the weights obtained, say, via the \emph{gradient descent} algorithm; but applies to \emph{any} weights (subject to the aforementioned assumptions) achieving a small training loss.

Moreover, the upper bound established in Theorem~\ref{thm:sigmoid} is also oblivious to the number $\overline{m}$ of the neurons of the \NN used for fitting. In particular, adopting a teacher/student setting as in~\cite{goldt2019dynamics} where the input/label pairs $(X_i,Y_i)$ are generated by a teacher \texttt{NN}; the output norm of any student \texttt{NN}---which may potentially be significantly overparameterized with respect to the teacher \texttt{NN}---is still well-controlled, provided the assumptions of Theorem~\ref{thm:sigmoid} are satisfied. The extra requirement that $w^T X$ is symmetric is quite mild: it holds for many data distributions, e.g., for $X\distr \mathcal{N}(0,\Sigma)$ where $\Sigma$ is an arbitrary positive semidefinite matrix. 

The $o_N(1)$ term is due to a certain high probability event $\mathcal{E}_0$, see~\eqref{eq:event0} in the proof. The probability of this event is controlled through the weak law of large numbers; and the $o_N(1)$ term can be improved explicitly (a) to $O(1/N)$ if $\mathbb{E}[Y^2]<\infty$; and (b) to $\exp(-\Theta(N))$ if $Y_i$ satisfy the large deviations bounds (which holds, for instance, when the moment generating function of $Y_i$ exists in a neighbourhood around zero). %admits a moment generating function in a neighbourhood around zero.
Moreover, if $Y$ is (almost surely) bounded (which holds for real data sets, as noted earlier), then it can be dropped altogether.

%An important remark pertains to the generalization ability of such networks. Bartlett established in~\cite[Theorem~17]{bartlett1998sample} that for sigmoid networks with bounded outer norm, a certain quantity called the \emph{fat-shattering dimension}---which is essentially a scale-sensitive measure of the complexity of class to which the network belongs to---is small, which, in turn, implies small generalization error. Hence, it appears that by leveraging our result, it is possible to give good generalization bounds.

Furthermore, Corollary~\ref{coro:sigmoid}---which follows immediately from Theorem~\ref{thm:sigmoid}---asserts that even under the mild assumption $R=\exp\left(d^{O(1)}\right)$ (i.e., the weights $w_j$ are \emph{unbounded} from a practical perspective), $\sum_j a_j$ is still $O(1)$, provided that that the number $N$ of data is polynomial in $d$. %(with a low-degree). 

Moreover, an inspection of the proof of Theorem~\ref{thm:sigmoid} reveals the following. 
The constant $3(1+e)$ can be improved to any constant greater than four with slightly more work.  Moreover, the thesis of Theorem~\ref{thm:sigmoid} still remains valid (with appropriately modified constants) for any non-negative activation which is continuous at the origin and whose value at the origin is positive. This includes the softplus activation $\ln\left(1+e^x\right)$~\cite{glorot2011deep}, the Gaussian activation, $\exp(-x^2)$; among others. 
\subsection{Self-Regularity for the ReLU Networks}\label{sec:main-bd-ReLU}
Our next focus is on the ReLU networks. This object, for each input $X\in\R^d$, computes the function~\eqref{eq:NN-fnc} with $\sigma(x) =\ReLU(x) =\max\{x,0\}=\frac12(x+|x|)$. 
%$\left(a,W\right)\in\R_{\ge 0}^{\overline{m}}\times \R^{\overline{m}\times d}$ consisting of $\overline{m}$ neurons, and ReLU activation, $\ReLU(x)=\max\{x,0\}=\frac12(x+|x|)$. This \NN computes, for each $X\in\R^d$, the function $\sum_{1\le j\le \overline{m}}a_j\ReLU\left(w_j^T X\right)$. 

We first observe that the \texttt{ReLU }function is positive homogeneous: for any $c\ge 0$ and $x\in\mathbb{R}$, $\ReLU(cx)=c\cdot\ReLU(x)$. For this reason, we may assume, without loss of generality, that $\|w_j\|_2 =1$ for $1\le j\le \overline{m}$. Indeed, if $w_j\ne 0$, one can simply ``push" its norm outside; whereas if $w_j=0$, then one can replace it with any unit norm vector and set $a_j=0$ instead. 

It is worth noting that since the ReLU case requires no explicit assumptions on $\|w_j\|_2$, an outer bound for this case is a somewhat stronger conclusion than an outer bound for the case of sigmoid activation.

Equipped with this, we now present our next result. 
\begin{theorem}\label{thm:ReLU}
Let $\delta,M>0$; and $\left(X_i,Y_i\right)\in\R^d\times \R$, $i\in[N]$ be i.i.d. data with $\mathbb{E}\left[\left|Y_i\right|\right]=M<\infty$; where $N$ satisfies Assumption~\ref{sump:N-exp-d}. For any $\overline{m}\in\mathbb{N}$, define \[
\mathcal{G}\left(\overline{m},\delta\right)=\left\{\left(a,W\right)\in \R_{\ge 0}^{\overline{m}}\times \R^{\overline{m}\times d}:\|w_j\|_2=1, 1\le j\le \overline{m};\,\,\EmpRisk{a,W}\le \delta^2\right\},
\]
where $\EmpRisk{\cdot}$ is defined in \eqref{eq:training-error} with $\sigma(\cdot)=\ReLU(\cdot)$. Suppose, in addition, that for $Y_w\triangleq w^TX$, (a) there exists a $\boldsymbol{\mu^*}>0$ such that $\mathbb{E}\left[\ReLU(Y_w)\right]\ge\boldsymbol{\mu^*}$ for any $w\in B_2(0,1)$; and (b) for some $s>0$, $M_1(s)$ and $M_2(s)$ are independent of $d$ and are finite; where $M_1(s)\triangleq \sup_{w:\|w\|_2=1}M_{Y_w}(s)$ and $M_2(s) \triangleq \sup_{w:\|w\|_2=1}M_{Y_w}(-s) $. %the values $M_1(s)$ and $M_2(s)$ defined by $\sup_{w:\|w\|_2=1}M_{Y_w}(s)\triangleq M_1(s)$ and $\sup_{w:\|w\|_2=1}M_{Y_w}(-s)\triangleq M_2(s)$
%for some $s>0$, $\sup_{w:\|w\|_2=1} M_w(s)\triangleq \sup_{w:\|w\|_2=1} \mathbb{E}[\exp(sY_w)]=M(s)<\infty$, and $M(s)$ is independent of $d$. 
Then, 
\begin{equation}\label{eq:ReLU-pb-bd}
\mathbb{P}\left(\sup_{(a,W)\in \mathcal{G}\left(\delta\right)}\|a\|_1 \le 4(\delta+2M)(\boldsymbol{\mu^*})^{-1}\right)\ge  1-\left(\frac{12\sqrt{Cd}}{\boldsymbol{\mu^*}}\right)^d\exp\left(-\Theta(N)\right)-N\exp\left(-\Theta(d)\right)-o_N(1),
\end{equation}
where $\mathcal{G}(\delta)\triangleq \bigcup_{\overline{m}\in\mathbb{N}}\mathcal{G}\left(\overline{m},\delta\right)$.
\end{theorem}
The proof of Theorem~\ref{thm:ReLU} is provided in Section~\ref{sec:pf-ReLU}.

In particular, it suffices to have a near-linear number of samples, $N=\Theta(d\log d)$, to obtain a good, uniform, control over $\|a\|_1$. As mentioned above,  we managed to bypass the dependence on the term $R$ that appears in Theorem~\ref{thm:sigmoid} by leveraging the fact that \texttt{ReLU} is a positive homogenenous function. 

Analogous to Theorem~\ref{thm:sigmoid}, the bound established in Theorem~\ref{thm:ReLU} is also oblivious to (a) how the training is done, and (b) the number $\overline{m}$ of neurons. In particular, even potentially overparameterized networks have a well-controlled outer norm; provided that they achieve a small training error on a sufficient number $N$ of data. The additional distributional requirements are still mild. For instance, when $X\distr \mathcal{N}(0,I_d)$, $w^TX\distr \mathcal{N}(0,1)$ for any $w$ with $\|w\|_2=1$; and $\boldsymbol{\mu^*}$ can be taken to be $1/\sqrt{2\pi}$. The requirement (b) ensures the existence of the moment generating function in a neighborhood around zero, hence the large deviations bounds are applicable. The same remarks on $o_N(1)$ term following Theorem~\ref{thm:sigmoid} also apply here: it can be improved to $O(1/N)$ or $\exp(-\Theta(N))$ under slightly stronger assumptions on $Y_i$.

\subsection{Self-Regularity for the Step Networks}\label{sec:main-bd-Step}
Our final focus is on the step networks. This object, for each $X\in\R^d$, computes~\eqref{eq:NN-fnc} with $\sigma(x)=\Step(x) = \ind\{x\ge 0\}$. 

Like the ReLU case, $\Step(x)$ is also homogeneous: for every $c\ge 0$, $\Step(cx)=\Step(x)$. For this reason, we assume, without loss of generality, $\|w_j\|_2 =1$, $1\le j\le \overline{m}$. 
\begin{theorem}\label{thm:Step}
Let $\delta,M>0$; and $\left(X_i,Y_i\right)\in\R^d\times \R$, $i\in[N]$ be i.i.d. data with $\mathbb{E}\left[\left|Y_i\right|\right]=M<\infty$; where $N$ satisfies Assumption~\ref{sump:N-exp-d}. For any $\overline{m}\in\mathbb{N}$, define
\[
\mathcal{H}\left(\overline{m},\delta\right)=\left\{\left(a,W\right)\in \R_{\ge 0}^{\overline{m}}\times \R^{\overline{m}\times d}:\|w_j\|_2=1, 1\le j\le \overline{m};\,\,\EmpRisk{a,W}\le \delta^2\right\},
\]
with $\EmpRisk{\cdot}$ as in \eqref{eq:training-error} with $\sigma(\cdot)=\Step(\cdot)$. Moreover, assume that for some $\eta>0$, $\inf_{w:\|w\|_2=1}\mathbb{P}\left(w^T X\ge \eta\right)\ge \eta$. Then,%, for every $w$ with $\|w\|_2=1$. Then,
\[
\mathbb{P}\left(\sup_{(a,W)\in \mathcal{H}\left(\delta\right)}\|a\|_1 \le 2(\delta+2M)\eta^{-1}\right)\ge 1-\left(\frac{6\sqrt{Cd}}{ \eta}\right)^d\exp\left(-\Theta(N)\right)-N\exp\left(-\Theta(d)\right)-o_N(1)
\]
where $
\mathcal{H}(\delta)\triangleq \bigcup_{\overline{m}\in\mathbb{N}}\mathcal{H}\left(\overline{m},\delta\right)$.
\end{theorem}
The proof of Theorem~\ref{thm:Step} is provided in Section~\ref{sec:pf-Step}.

Main remarks following Theorems~\ref{thm:sigmoid} and \ref{thm:ReLU}---in particular, independence from $\overline{m}$ as well as the training algorithm---apply here, as well. 

The extra condition on the distribution ensures that the collection $\left\{\mathbb{P}(w^T X\ge \eta):\|w\|_2=1\right\}$ is uniformly bounded away from zero. This is again quite mild, as demonstrated by the following example. Suppose $Y_w\triangleq w^TX$ is centered and equidistributed for $w$ with $\|w\|_2=1$. (Observe that this is indeed the case, e.g. when $X\distr \mathcal{N}(0,I_d)$.)  Then as long as ${\rm Var}(Y_w)>0$ the extra requirement per Theorem~\ref{thm:Step} is satisfied. Indeed, for this case $\mathbb{P}(Y_w>0)>0$. Hence, using the continuity of probabilities
\[
\mathbb{P}(Y_w>0)=\mathbb{P}(w^T X>0)=\lim_{t\to\infty}\mathbb{P}\left(w^T X>t^{-1}\right)>0,
\]
one ensures the existence of such an $\eta$. In the case where $X\distr\mathcal{N}(0,I_d)$, one can concretely take $\eta=0.3$.

\section{Generalization Guarantees via Outer Norm Bounds}\label{sec:genel}
\subsection{The Learning Setting}% and a Fat-Shattering Dimension Result}
In this section, we leverage the outer norm bounds we established in Theorems~\ref{thm:sigmoid}-\ref{thm:Step} to provide generalization guarantees for the neural network architectures having non-negative output weights that we investigated. %we investigated.

Our approach is through a quantity called the \emph{fat-shattering dimension} (FSD) of such networks introduced by  Kearns and Schapire~\cite{kearns1994efficient}. This quantity is essentially a scale-sensitive measure of the complexity of the ``class" (appropriately defined) that the network architecture being considered belongs to. We introduce the FSD formally in Definition~\ref{def:fsd} found in Section~\ref{sec:pf-main-genel}. For more information on the FSD, we refer the interested reader to the original paper by  Kearns and Schapire~\cite{kearns1994efficient}; as well as earlier papers by Bartlett, Long, and Williamson~\cite{bartlett1996fat}, and Bartlett~\cite{bartlett1998sample}. 

In what follows, we prove our promised generalization guarantee (Theorem~\ref{thm:main-generalization} below) by combining the prior results on the FSD of such networks with our outer norm bounds. Bartlett provides in~\cite{bartlett1998sample} upper bounds on the FSD of certain function classes $H$. He then leverages these bounds to give good generalization guarantees. One of the classes he studies is precisely the class of two-layer \NN with a {\bf bounded outer norm} (as we do). In particular, he establishes in~\cite[Corollary~24]{bartlett1998sample} (which is restated as Theorem~\ref{thm:bartlett} below) that the class of two-layer networks with {\bf bounded outer norm} has a well-controlled FSD: informally, it has ``low complexity".  He then leverages the FSD bounds to devise good generalization guarantees for the architectures that he investigates. It is worth noting, however, that he establishes this link in the context of \emph{classification} setting, $Y\in\{\pm 1\}$. Since we assume $a_j\ge 0$, and the activations we study are non-negative, this does not apply to our case: the outputs of the networks we study are always non-negative. Nevertheless, we by-pass this by combining our outer norm bounds (Theorems~\ref{thm:sigmoid}-\ref{thm:Step}), Theorem~\ref{thm:bartlett}, as well as building upon several other prior results tailored for the \emph{regression} setting. %(see below). 

We next recall the learning setting for convenience. Let $\mathcal{D}$ be a distribution on $\R^d\times \R$ for the input/label pairs $(X,Y)$; and let $(X_i,Y_i)\sim\mathcal{D}$, $1\le i\le N$, be the i.i.d.\,training data. The goal of the learner is to find a \NN $(a,W)\in\R^{\overline{m}}\times \R^{\overline{m}\times d}$ with $\overline{m}$ hidden units (neurons) and activation $\sigma(\cdot)$ which ``explains" the data $(X_i,Y_i)$, $1\le i\le N$, as accurately as possible, often by solving the \emph{empirical risk minimization} problem, $\min_{a,W}\EmpRisk{a,W}$~\eqref{eq:training-error}. The ``learned" network is then used for predicting the unseen data. The generalization ability of the ``learned" network $(a,W)\in \R^{\overline{m}}\times \R^{\overline{m}\times d}$ is quantified by the so-called \emph{generalization error} (also known as the \emph{population risk})
\begin{equation}\label{eq:pop-risk}
\mathcal{L}(a,W)\triangleq \mathbb{E}_{(X,Y)\sim \mathcal{D}}\Bigl[\Bigl(Y-\sum_{1\le j\le \overline{m}}a_j \sigma\left(w_j^T X\right)\Bigr)^2\Bigr].
\end{equation}
Here, the expectation is taken w.r.t. to a fresh sample $(X,Y)\sim \mathcal{D}$, which is independent of the training data. The ``gap" \[
\left|\EmpRisk{a,W}-\mathcal{L}(a,W)\right|  = \left|\frac1N\sum_{1\le i\le N}\left(Y_i - \sum_{1\le j\le \overline{m}}a_j\sigma\left(w_j^T X_i\right)\right)^2  - \mathbb{E}_{(X,Y)\sim \mathcal{D}}\Bigl[\Bigl(Y-\sum_{1\le j\le \overline{m}}a_j \sigma\left(w_j^T X\right)\Bigr)^2\Bigr] \right|
\] between the training error and the generalization error is called the \emph{generalization gap}.

In what follows, we focus our attention on the generalization ability~\eqref{eq:pop-risk} of the learned networks $(a,W)$ that achieved a small training error, $\EmpRisk{a,W}\le \delta^2$~\eqref{eq:training-error}, on a polynomial (in $d$) number of data. The details of the training process (such as the  algorithm used for training) are immaterial to us; and our results apply to {\bf any} \NN $(a,W)$ provided it achieved a small training error, $\EmpRisk{a,W}\le \delta^2$. 

In this section, we also assume that the labels $Y$ are bounded: $\mathcal{D}$ is such that for some $M>0$, $|Y|\le M$ almost surely. This is necessary, as the prior results we employ from Haussler~\cite{haussler1992decision} and Bartlett, Long, and Williamson~\cite{bartlett1996fat} (in particular, see Theorem~\ref{thm:haussler}) apply only to the case where the labels are bounded. For this reason, the $o_N(1)$ terms present in Theorems~\ref{thm:sigmoid}-\ref{thm:ReLU} disappear, see the remarks following each theorem. 

\subsection{The Generalization Guarantees}
Equipped with our outer norm bounds (Theorems~\ref{thm:sigmoid}-\ref{thm:Step}) and Theorem~\ref{thm:bartlett}, we now establish the promised generalization guarantees for the aforementioned networks whose output weights $a_i$ are non-negative. To that end, let $\alpha,M,\mathcal{M},A>0$ be certain parameters (elaborated below); and set
\begin{equation}\label{eq:xi-in-prop}
\xi\left(\alpha,M,\mathcal{M},A\right)\triangleq \frac{2}{\ln 2}\cdot \frac{c \cdot 128^2 \cdot \mathcal{M}^6 A^6\cdot \max\{\mathcal{M}A,2M\}^2 }{ \alpha^2}\cdot  \ln \left(\frac{128 \mathcal{M}^3A^3 \max\{\mathcal{M}A,2M\}}{\alpha}\right),
\end{equation}
where $c>0$ is the absolute constant appearing in Theorem~\ref{thm:bartlett}.
Our result is as follows.
\begin{theorem}\label{thm:main-generalization}
Let $\alpha,\delta,M,R>0$, and $(X_i,Y_i)\in\R^d \times \R$, $1\le i\le N$, be i.i.d.\,samples drawn from an arbitrary distribution $\mathcal{D}$ on $\R^d\times \R$ with $|Y|\le M$ almost surely; where $N$ satisfies Assumption~\ref{sump:N-exp-d}. %For the absolute constant $c>0$ appearing in Theorem~\ref{thm:bartlett}, set
For the $\xi$ term defined in~\eqref{eq:xi-in-prop}, set 
\begin{equation}\label{eq:prob-term}
   \zeta(\alpha,M,A,N)\triangleq \exp\left(\xi(\alpha,M,2,A)\cdot d \cdot \ln^2 \left(\frac{2304 \cdot N \cdot A^2 \cdot  \max\{2A,M\}}{\alpha}\right) - \frac{\alpha^2 \cdot  N}{64\cdot \max\{2A,M\}^2}\right).
\end{equation}
\begin{itemize}
    \item[(a)] {\bf (Sigmoid Networks)} Under the assumptions of Theorem~\ref{thm:sigmoid}, with probability at least
    \[
    1-\zeta(\alpha,M,3(1+e)(\delta+2M),N) - \left(3R\sqrt{Cd}\right)^d\exp\left(-\Theta(N)\right)-N\exp\left(-\Theta(d)\right)
    \]
    over $(X_i,Y_i)\sim\mathcal{D}$, $1\le i\le N$, it holds that
    \[
    \sup_{(a,W) \in \mathcal{S}(\delta,R)} \mathbb{E}_{(X,Y)\sim \mathcal{D}}\left[\left(Y-\sum_{1\le j\le \overline{m}}a_j \SGM\left(w_j^T X\right)\right)^2\right]\le \alpha+\delta^2,
    \]
    provided
    \[
    N \ge c \cdot 2^{21} \cdot \frac{A^6\cdot \max\{A,M\}^2}{\alpha^2} \cdot d \quad\text{and}\quad \alpha \le 2^{11}\cdot A^3\cdot \max\{A,M\}
    \]
    with $A=3(1+e)(\delta+2M)$. Here, $\mathcal{S}(\delta,R)$ is the set introduced in Theorem~\ref{thm:sigmoid}.
    \item[(b)] {\bf (ReLU Networks)} Under the assumptions of Theorem~\ref{thm:ReLU} and assuming additionally $\boldsymbol{\mu^*} = \exp\Bigl(o(d)\Bigr)$,
  with probability at least
     \[
 1-\zeta\left(\alpha,M,\frac{4\sqrt{Cd}(\delta+2M)}{\boldsymbol{\mu^*}},N\right) - \left(\frac{12\sqrt{Cd}}{\boldsymbol{\mu^*}}\right)^d\exp\left(-\Theta(N)\right)-N\exp\left(-\Theta(d)\right),
\]
over $(X_i,Y_i)\sim\mathcal{D}$, $1\le i\le N$, it holds that
     \[
\sup_{(a,W)\in \mathcal{G}(\delta)}  \mathbb{E}_{(X,Y)\sim \mathcal{D}}\left[\left(Y-\sum_{1\le j\le \overline{m}}a_j \ReLU\left(w_j^T X\right)\right)^2\right]\le \alpha+\delta^2+e^{-\Theta(d)},
\]
 provided
    \[
    N \ge c \cdot 2^{21} \cdot \frac{A^6\cdot \max\{A,M\}^2}{\alpha^2} \cdot d \quad\text{and}\quad \alpha \le 2^{11}\cdot A^3\cdot \max\{A,M\}
    \]
    with $A=\frac{4\sqrt{Cd}(\delta+2M)}{\boldsymbol{\mu^*}}$.
    Here, $\mathcal{G}(\delta)$ is the set introduced in Theorem~\ref{thm:ReLU}.
    \item[(c)] {\bf (Step Networks)} Under the assumptions of Theorem~\ref{thm:Step}, with probability at least
    \[
    1-\zeta\left(\alpha,M,\frac{2(\delta+2M)}{\eta},N\right)-\left(\frac{6\sqrt{Cd}}{ \eta}\right)^d\exp\left(-\Theta(N)\right)-N\exp\left(-\Theta(d)\right)
    \]
    over $(X_i,Y_i)\sim \mathcal{D}$, $1\le i\le N$, it holds that
    \[
     \sup_{(a,W) \in \mathcal{H}(\delta)} \mathbb{E}_{(X,Y)\sim \mathcal{D}}\left[\left(Y-\sum_{1\le j\le \overline{m}}a_j \Step\left(w_j^T X\right)\right)^2\right]\le \alpha+\delta^2,
    \]
     provided
    \[
    N \ge c \cdot 2^{21} \cdot \frac{A^6\cdot \max\{A,M\}^2}{\alpha^2} \cdot d \quad\text{and}\quad \alpha \le 2^{11}\cdot A^3\cdot \max\{A,M\}
    \]
    with $A=\frac{2(\delta+2M)}{\eta}$.
        Here, $\mathcal{H}(\delta)$ is the set introduced in Theorem~\ref{thm:Step}.

    \end{itemize}
\end{theorem}

%It appears that the outer norm bounds we established can be leveraged to yield generalization guarantees. We describe here our reasoning and an ongoing work towards this direction. %It appears he bounds established in Theorems~\ref{thm:sigmoid}-\ref{thm:Step} can be leveraged to yield generalization guarantees; 

% Bartlett provides in~\cite{bartlett1998sample} upper bounds on the \emph{fat-shattering dimension} (FSD) of certain function classes $H$---a scale-sensitive measure of the complexity of $H$ (see~\cite{bartlett1998sample} for a definition of FSD). He then leverages these bounds to give good generalization guarantees. One of the classes he studies is precisely the class of two-layer \NN with {\bf bounded outer norm}; and he  establishes in~\cite[Corollary~24]{bartlett1998sample} the following:

Theorem~\ref{thm:main-generalization} is established by combining various individual results established in separate works~\cite{haussler1992decision,bartlett1996fat,alon1997scale,bartlett1998sample} together with our outer norm bounds. See Section~\ref{sec:pf-main-genel} for its proof.

We next comment on the performance parameters appearing in Theorem~\ref{thm:main-generalization}. The parameter $\alpha$ controls the so-called \emph{generalization gap}: the gap between the training error and the generalization error. The parameter $\delta$ controls the training error: we study those $(a,W)$ with $\EmpRisk{a,W}\le \delta^2$. The parameter $M$ is an (almost sure) upper bound on the labels; whereas $R$ is an (quite mild) upper bound on internal weights required for the technical reasons, only for the case of sigmoid networks, see Theorem~\ref{thm:sigmoid}, Corollary~\ref{coro:sigmoid}; and the remarks following them.

The term $\zeta(\alpha,M,A,N)$ is a probability term appearing in the uniform convergence result (Proposition~\ref{prop:main}) that we employ. This proposition provides a control for the generalization gap uniformly over all two-layer neural networks with \emph{bounded outer norm} like we investigate herein. 

It is worth noting that in the regime $d\to \infty$ (which is a legitimate assumption for many existing guarantees in the field of machine learning) and for $N\ge d^{O(1)}$; $\xi(\alpha,M,A) = O(1)$ (with respect to $d$), provided that $A=O(1)$ like we establish earlier. Namely, this object is simply a constant in $d$. Furthermore, while we made no attempts in simplifying it, it can potentially be improved. In the sigmoid and step cases, the value of $A$ that we consider is indeed $O(1)$. For the ReLU case, however, the situation is more involved; and a certain scaling which makes $A={\rm poly}(d)$ is necessary, as we elaborate soon. Soon in Section~\ref{sec:sample-komple}, we investigate the probability term $\zeta(\alpha,M,A,N)$ appearing in~\eqref{eq:prob-term}. We show that provided $N$ is sufficiently large (while remaining polynomial in $d$), the $\zeta$ term behaves like $\exp\left(-d^{O(1)}\right)$, thus it is indeed $o_d(1)$. Moreover, our analysis will also reveal that the dependence of $N$ on $d$ is quite mild; and is in fact near-linear in some important cases of interest.

In particular, the probability term $\zeta(\alpha,M,A,N)$ is $o_d(1)$ provided 

Theorem~\ref{thm:bartlett} as well as the uniform generalization gap guarantee, Proposition~\ref{prop:main}, apply to activations with a bounded output; whereas the output of ReLU is potentially unbounded. In our proof, we bypass this by considering an auxiliary activation $\SReLU(\cdot)$, which is a ``saturated" version of the ReLU. Specifically, we let $\SReLU(x) = 0$ for $x\le 0$, $\SReLU(x)=x$ for $0<x\le 1$; and $\SReLU(x)=1$ for $x\ge 1$. We then rescale $w_j$ to have $\|w_j\|_2 = 1/\sqrt{Cd}$ and multiply  $A$ by $\sqrt{Cd}$ (we therefore consider  $A=4\sqrt{Cd}(\delta+2M)/\boldsymbol{\mu^*}$, $\sqrt{Cd}$ times the bound appearing in Theorem~\ref{thm:ReLU}). Note that this step is indeed valid due to the homogeneity of the ReLU activation, see also Section~\ref{sec:main-bd-ReLU}. Since $\|X\|_2 \le \sqrt{Cd}$ with probability at least $1-\exp(-\Theta(d))$ and since $|w_j^T X|\le 1$ for $\|w_j\|_2 =1/\sqrt{Cd}$ and $\|X\|_2\le \sqrt{Cd}$ by Cauchy-Schwarz inequality; the output of this activation will, w.h.p., coincide with that of the ReLU activation. We then control the difference between the generalization errors for a pair of two-layer neural networks having the same architecture, the same number $\overline{m}\in\mathbb{N}$ of hidden units, the same weights $(a,W)$; but different activations (one with $\ReLU(\cdot)$ and the other with  $\SReLU(\cdot)$). This done by a conditioning argument. See the proof for further details. 
%conditional expectation argument, where the upper bound~\eqref{eq:distributional-lambda} on $\lambda(d)$ ensures proper behavior. See the proof for further details. 
 
Similar to what we have noted previously for our outer norm bounds, Theorem~\ref{thm:main-generalization} is also oblivious to (a) how the training is done and (b) the number $\overline{m}$ of hidden units as long as $a_i\ge 0$, and $\EmpRisk{a,W}\le \delta^2$ for the learned network. Moreover, similar to prior cases, the extra conditional expectation requirement~\eqref{eq:distributional-lambda} is quite mild.

Our next focus is on the sample complexity required by Theorem~\ref{thm:main-generalization}. We show that they are indeed polynomial in $d$. Furthermore for some very important cases, they are even near-linear.
\subsection{Sample Complexity Analysis}\label{sec:sample-komple} While the required sample complexity $N$ can simply be inferred from Theorem~\ref{thm:main-generalization}, we spell out the implied scaling analysis below for convenience. In what follows, all asymptotic notations are w.r.t. the natural parameter $d$ (namely the dimension) of the problem in the regime $d\to\infty$; and our goal is to ensure that the corresponding probability term is $1-o_d(1)$ for an appropriate function $o_d(1)$. (It is worth noting though that our bounds will be in fact much stronger, e.g. $1-\exp(-d^{O(1)})$.)

To that end, recall the term~\eqref{eq:xi-in-prop} with $\mathcal{M}=2$ appearing in Theorem~\ref{thm:main-generalization}:
\begin{equation}\label{eq:xi-main-thm}
\xi\left(\alpha,M,2,A\right)=\frac{2^{23} \cdot c \cdot A^6\cdot \max\{A,M\}^2 }{\ln 2\cdot  \alpha^2}\cdot  \ln \left(\frac{2^{11} \cdot A^3 \cdot\max\{A,M\}}{\alpha}\right).
\end{equation}
\subsubsection*{Sigmoid and Step Networks}
First, the outer norm bounds we establish indicate $A=O(1)$. Hence, the ``$A$ parameter" considered in parts (a) and (c) of Theorem~\ref{thm:main-generalization} are $O(1)$. Moreover, $M=O(1)$ (since it is not sound for the real-valued label $Y$ to grow with dimension $d$). Treating $\alpha$ as a constant in $d$, we then obtain $\xi(\alpha,M,A)=O(1)$ for the term appearing in~\eqref{eq:xi-main-thm}. Hence, in order to ensure that the probability term $\zeta$ appearing in~\eqref{eq:prob-term} is $o_d(1)$, a necessary and sufficient condition is $N = \Omega\left(d\ln^2 N\right)$. We claim that it suffices to have
\begin{equation}\label{eq:N-omega-d-ln-kup}
    N = \Omega\left(d \ln^2 d\right).
\end{equation}
Indeed, if $N$ satisfies~\eqref{eq:N-omega-d-ln-kup}, then provided $N$ remains polynomial in $d$, $N ={\rm poly}(d)$, it holds that 
\[
\ln^2 N = O\left(\ln^2 d\right) \implies d\ln^2 N = O\left(d\ln^2 d\right) = O(N).
\]
We now investigate the sample complexity required by the corresponding outer norm bounds for the case of sigmoid and step networks. 
\paragraph{ Sigmoid networks.} Note, in this case, that the dominant contribution to the probability term appearing in Theorem~\ref{thm:sigmoid}/Theorem~\ref{thm:main-generalization}(a) (other than $\xi$ term) is $(3R\sqrt{Cd})^d\exp(-\Theta(N))$. Suppose first that $R=d^K$ where $K=O(1)$ (namely $R$ remains polynomial in $d$). Then
\begin{align*}
    (3R\sqrt{Cd})^d\exp(-\Theta(N)) &= \exp\Bigl(-\Theta(N) +d\left(K+\frac12\right)\ln d + d\ln (3\sqrt{C})\Bigr) \\
    &= \exp\Bigl(-\Theta(N) +\Theta(d\ln d) + o(d\ln d)\Bigr).
\end{align*}
provided $N=\Omega(d\ln d)$, this bound is indeed $o_d(1)$. Taking the maximum between this and~\eqref{eq:N-omega-d-ln-kup}, we obtain that it suffices to have $N=\Omega(d\ln^2 d)$, which is near-linear. 

Suppose next that $R=\exp(d^K)$, like in Corollary~\ref{coro:sigmoid}. Then provided $K>0$,
\begin{align*}
    (3R\sqrt{Cd})^d\exp(-\Theta(N)) &= \exp\Bigl(-\Theta(N) +d^{K+1} + \frac12 d\ln d + d\ln (3\sqrt{C})\Bigr) \\
    &= \exp\Bigl(-\Theta(N) +d^{K+1}+o\left(d^{K+1}\right)\Bigr).
\end{align*}
Hence, provided $N=\Omega(d^{K+1})$, this bound is indeed $o_d(1)$. Taking the maximum between this and~\eqref{eq:N-omega-d-ln-kup}, we obtain that it suffices to have $N=\Omega(d^{K+1})$, which is polynomial in $d$. 
\paragraph{ Step networks.} Treating the distributional parameter $\eta$ appearing in Theorem~\ref{thm:Step}/Theorem~\ref{thm:ReLU} as a constant in $d$, we have
\begin{align*}
\exp\left(-\Theta(N)\right) \left(\frac{6\sqrt{Cd}}{\eta}\right)^d & = \exp\left(-\Theta(N) + \frac12 d\ln d + d\ln\left(\frac{6\sqrt{C}}{\eta}\right)\right) \\
& = \exp\left(-\Theta(N)+\Theta(d\ln d) + o(d\ln d)\right).
\end{align*}
Thus, provided $N=\Omega(d\ln d)$, this bound is indeed $o_d(1)$. Taking the maximum between this and~\eqref{eq:N-omega-d-ln-kup}, we obtain that it suffices to have $N=\Omega(d\ln^2 d)$, which, again, is near-linear.
\subsubsection*{ReLU Networks}
The situtation is more involved for the case of ReLU networks. We first study the $\xi$ term~\eqref{eq:xi-main-thm}. Treating $M,\alpha,C,\delta,\boldsymbol{\mu^*} = O(1)$ (in $d$), 
\[
\xi\left(\alpha,M,2,\frac{4\sqrt{C}(\delta+2M)}{\boldsymbol{\mu^*}}\sqrt{d}\right) = \Theta\left(d^4 \ln d\right).
\]
Hence, 
\begin{align*}
    \zeta\left(\alpha,M,\frac{4\sqrt{Cd}(\delta+2M)}{\boldsymbol{\mu^*}},N\right) & = \exp\Bigl(\Theta\Bigl(d^4 \cdot \ln d \cdot d \cdot \ln^2 (Nd)\Bigr)-\Theta\left(\frac{N}{d}\right)\Bigr)  \\
    &= \exp\Bigl(\Theta\Bigl(d^5 \cdot \ln d \cdot \ln^2 (Nd)\Bigr)-\Theta\left(\frac{N}{d}\right)\Bigr)\\
    &=\exp\Bigl(\Theta\Bigl(d^5 \cdot \ln^3 d \Bigr)-\Theta\left(\frac{N}{d}\right)\Bigr),
\end{align*}
where we used the fact $\ln(Nd)=\Theta(\ln d)$ if $N={\rm poly}(d)$. Thus, provided $N=\Omega\left(d^6\ln^3 d\right)$, this bound is indeed $o_d(1)$. Inspecting next the term $(12\sqrt{Cd}/\boldsymbol{\mu^*})^d\exp(-\Theta(N))$ appearing in the probability bound, we observe as long as $N=\Omega(d\ln d)$, this term is also $o_d(1)$. Taking the maximum of these two, it suffices to have $N=\Omega\left(d^6\ln^3 d\right)$. This, again, is a polynomial in $d$; albeit having a slightly worse degree (of six). %We leave the improvement of this exponent of six as an open problem for future work.
\section{Conclusion and Future Directions}\label{sec:conclude}
%We established that the \emph{outer norm} of the aforementioned \NN models achieving a small training loss on a polynomially (in $d$) many data and having non-negative output weights is \emph{well-controlled}. Our results are independent of the width $\overline{m}$ of the network and the training algorithm; and are valid under very mild distributional assumptions on input/label pairs. 
We have studied two-layer \NN models with sigmoid, ReLU, and step activations; and established that the \emph{outer norm} of any such \NN achieving a small training loss on a polynomially (in $d$) many data and having non-negative output weights is \emph{well-controlled}. Our results are independent of the width $\overline{m}$ of the network and the training algorithm; and are valid under very mild distributional assumptions on input/label pairs. We then leveraged the outer norm bounds we established to obtain good generalization guarantees for the networks we investigated. Our generalization results are obtained by employing earlier results on the fat-shattering dimension of such networks, and have good sample complexity bounds as we have discussed. In particular, for certain important cases of interest, we obtain near-linear sample guarantees. 

We now provide future directions. As was already mentioned, our approach operates under mild distributional requirements; and can potentially handle different distributions as well as other activations, provided (rather natural) certain properties of these objects we leveraged remain in place. %Moreover, an ongoing work is to rigorously carry out the program outlined in Section~\ref{sec:genel}. 

%As noted previously, Bartlett established in~\cite{bartlett1998sample} that for certain networks with bounded outer norm, a quantity called the \emph{fat-shattering dimension} is bounded. This, in turn, implies good generalization. Part of an ongoing work is to blend our results with those in~\cite{bartlett1998sample} to give good generalization guarantees. 

%It appears that by leveraging the polynomial regression (PR) method like in~\cite[Theorem~3.7]{emschwiller2020neural}, our results imply good generalization guarantees for such networks under the ``teacher/student" setup. This is based on approximating both the teacher and student network with the same estimator obtained from the PR method; and  is part of an ongoing work.

%Above, we demonstrated that the non-negativity assumption is necessary in strict sense. Nevertheless, that example is rather specially tailored. Hence, it may still be possible to address the case of negative output weights under proper assumptions. 

A very important question is to which extent our approach applies to deeper networks. In what follows, we give a very brief argument demonstrating that for such an extension, one needs much more stringent regularity assumptions on the internal weights. Consider, as an example, a ReLU network with three hidden layers. Observe that the outputs of the neurons at the first hidden layer are non-negative as $\ReLU(x)\ge 0$ for all $x\in\R$. Let us now focus on its second hidden layer, which takes weighted sums of the outputs of the first hidden layer. If all the weights in the second layer are negative, then upon passing to ReLU, one obtains all zeroes, forcing the final output to be zero. Now, let us assume, instead, that the weights of the second layer are such that the input to the ReLU functions are positive, though arbitrarily close to zero (this can potentially be achieved, e.g., by taking many small negative weights and few large positive weights in a way that ensures proper cancellation). If this holds, then even if the outer norm, $\|a\|_1$, is very large, one still obtains a bounded output at the end of the network. As demonstrated by this conceptual example, one indeed needs more stringent assumptions on the internal weights so as to address larger depth. At the present time, we are unable to have a complete resolution of necessary and sufficient assumptions for addressing deeper architectures (while maintaining the position that these assumptions must also be sound from a practical point of view). %We plan to explore this direction in a future work.  

%As an example, consider a sigmoid network of depth $L$. Inspecting the proof of Theorem~\ref{thm:sigmoid}, the upper bound~\eqref{eq:sigmoid-important} still holds for a sum of form $\sum_i \sum_j a_j Z_{i,j}$, where $Z_{i,j}\ge 0$ is an ``intermediate output" obtained essentially by applying the sigmoid function $L-1$ times. On the other hand, the (i.i.d.) sum $\sum_{1\le i\le N}Z_{i,j}$ is $\Theta(N)$ provided $\mathbb{E}[Z_{i,j}]$ is $\Theta(1)$. Moreover, Hoeffding's inequality yields a good concentration for this sum, since $\SGM\left(\mathbb{R}\right)\subseteq [0,1]$. Hence, in this case  it is plausible to expect that $\|a\|_1$ is well-controlled (under a polynomial-in-$d$ sample complexity bound which depends also on the depth $L$). A similar reasoning applies also to the ReLU networks: the upper bound~\eqref{eq:ReLU-upper-bound} remains valid for a sum of form $\sum_i \sum_j a_j R_{i,j}$, where $R_{i,j}\ge 0$ is obtained by applying ReLU $L-1$ times. If $\mathbb{E}\left[R_{i,j}\right]$ grows in a benign manner in depth (in a precise sense), it is conceivable to expect a similar conclusion. We plan to address these in a future work.

Yet another important direction pertains to the non-negativity of the weights, and a crucial question is whether this assumption can be relaxed.  We now provide a brief argument demonstrating that in full generality, this is not necessarily the case. Namely, strictly speaking, the non-negativity assumption is necessary. We focus on the so-called ``teacher/student" setting, a setting that has been quite popular recently, see, e.g.~\cite{goldt2019dynamics}. In this setting, given i.i.d. input data $X_i\in\R^d$, $1\le i\le N$, a teacher network $\left(a^*,W^*\right)\in\R^{m^*}\times \R^{m^*\times d}$ with $m^*\in\mathbb{N}$ neurons and activation $\sigma(\cdot)$ generates the labels $Y_i$. That is, $Y_i = \sum_{1\le j\le m^*}a_j^* \sigma\left((w_j^*)^TX_i\right)$. A student network with an $\overline{m}\in\mathbb{N}$ number of hidden units (where $\overline{m}$ is not necessarily equal to $m^*$) is then ``trained" by minimizing the objective function~\eqref{eq:training-error} on the data $(X_i,Y_i)$, $1\le i\le N$; and the resulting network is then used for predicting the unseen data. We now construct a wider student network interpolating
the data whose vector of output weights has arbitrarily large norm, by introducing many cancellations. Fix $z\in\mathbb{N}$, a non-zero $v\in\R^d$; and $\nu>0$. Construct a new network $\left(\overline{a},\overline{W}\right)$ on $m^*+2z$ neurons as follows. Set $\overline{a_j}=a_j^*$ and $\overline{W_j}=W_j^*$ for $1\le j\le m^*$. For any $m^*+1\le j\le m^*+2z$, set $\overline{a_j}=\nu$ if $j$ is even, and $-\nu$, if $j$ is odd. At the same time, set $\overline{W_j}=v$ for $m^*+1\le j\le m^*+2z$. This network interpolates the data while $\|\overline{a}\|_1 = \|a^*\|_1 + 2z\nu$. Hence, $\|\overline{a}\|_1$ can be made arbitrarily large by amplifying $z$ and/or $\nu>0$. In particular, in full generality, such a non-negativity assumption is indeed necessary. It is worth noting, however, that the example above is a somewhat tailored one involving many dependencies/cancellations. It might still be possible to establish similar bounds for the case of potentially negative weights under more stringent constraints on them which prevent such cancellations. %We plan to explore this direction in a future work. 
\section{Proofs}\label{sec:pfs}
\subsection{Proof of Theorem~\ref{thm:sigmoid}}\label{sec:pf-sigmoid}

\begin{proof}[Proof of Theorem~\ref{thm:sigmoid}] Observe that
\begin{equation}\label{eq:event0}
    \mathbb{P}(\mathcal{E}_0)\ge 1-o_N(1)  \quad\text{for}\quad \mathcal{E}_0\triangleq \{\textstyle \sum_{i=1}^N |Y_i|\le 2MN\},
\end{equation}
%Let $\mathcal{E}_0\triangleq \left\{\sum_{1\le i\le N}|Y_i|\le 2MN\right\}$. Then by the weak law of large numbers, $\mathbb{P}(\mathcal{E}_0)\ge 1-o_N(1)$. 
%Observe that using Cauchy-Schwarz $\mathbb{E}[|Y_i|]\le \mathbb{E}[Y_i^2]^{1/2}=M$. Hence, $\mathbb{P}(\mathcal{E}_0^c)\le \mathbb{P}\left(\sum_{1\le i\le N}\Bigl(|Y_i|-\mathbb{E}[|Y_i|]\Bigr)\ge NM\right)\le \mathbb{P}\left(\left|\sum_{1\le i\le N}\Bigl(|Y_i|-\mathbb{E}[|Y_i|]\Bigr)\right|\ge NM\right)=O(1/N)$, using Chebyshev's inequality and the independence of $Y_i$, $i\in[N]$.
using the weak law of large numbers~\cite[Thm~2.2.14]{durrett2019probability}. Next, let $(a,W)\in\mathcal{S}\left(\delta,R\right)$. Then there exists an $\overline{m}\in\mathbb{N}$ such that $(a,W)\in \mathcal{S}\left(\overline{m},\delta,R\right)$. Applying Cauchy-Schwarz inequality, $
\sum_{1\le i\le N}\left|Y_i-\sum_{1\le j\le \overline{m}}a_j\SGM\left(w_j^T X_i\right)\right|\le N\delta$. Next, by the triangle inequality and the fact $\sum_i |Y_i|\le 2MN$ on $\mathcal{E}_0$,
\begin{equation}\label{eq:sigmoid-important}
   \sum_{1\le i\le N}\sum_{1\le j\le \overline{m}}a_j \SGM\left(w_j^T X_i\right)\le N(\delta+2M),
\end{equation}
on the event $\mathcal{E}_0$. Now, let $\mathcal{N}_\epsilon$ be an $\epsilon-$net for $B_2(0,R)$, $\epsilon>0$ to be tuned appropriately. Using Theorem~\ref{lemma:card-net}, one can ensure $\left|\mathcal{N}_\epsilon\right|\le (3R/\epsilon)^d$. Next, fix any $\widehat{w}\in\mathcal{N}_\epsilon$, and set $\overline{Z_i}\triangleq \widehat{w}^T X_i$, $1\le i\le N$. Since $\overline{Z_i}$ is symmetric, %it follows $\overline{Z_i}$ and $-\overline{Z_i}$ have the same distribution. Now, $\mathbb{P}(\overline{Z_i}\ge 0)=\mathbb{P}(-\overline{Z_i}\le 0)=\mathbb{P}(\overline{Z_i}\le 0)$. Consequently,
$\mathbb{P}(\overline{Z_i}\ge 0) = \mathbb{P}(-\overline{Z_i}\le 0)=\mathbb{P}(\overline{Z_i}\le 0)$, implying
$\mathbb{P}(\overline{Z_i}\ge 0)\ge \frac12$. Define now $Z_i\triangleq \ind\left\{\overline{Z_i}\ge 0\right\}$. Since $Z_i$ ``stochastically dominates" ${\rm Bernoulli}(1/2)$, we have
$\mathbb{P}\left(\sum_{1\le i\le N}Z_i \ge N/3\right)\ge \mathbb{P}\left({\rm Binomial}\left(N,1/2\right)\ge N/3\right)\ge 1-\exp\left(-\Theta(N)\right)$. The last inequality is due to standard large deviations bounds. Taking a union bound over the net $\mathcal{N}_\epsilon$, we obtain
\begin{equation}
\label{eq:event1}
   \textstyle \mathbb{P}(\mathcal{E}_1)\ge 1-(3R/\epsilon)^d\exp\left(-\Theta(N)\right),\quad\text{where}\quad \mathcal{E}_1\triangleq \bigcap_{\widehat{w}\in\mathcal{N}_\epsilon}\left\{\sum_{1\le i\le N}\ind\left\{\widehat{w}^T X_i\ge 0\right\}\ge N/3\right\}.
\end{equation}
Furthermore, another union bound over the data yields
\begin{equation}
\label{eq:event2}
\textstyle\mathbb{P}(\mathcal{E}_2)\ge 1-N\exp\left(-\Theta(d)\right),\quad \text{ where  }\quad \mathcal{E}_2\triangleq \left\{\|X_i\|_2^2\le Cd,1\le i\le N\right\}.
\end{equation}
We now choose $\epsilon=1/\sqrt{Cd}$. We claim that on the event $\mathcal{E}_1\cap \mathcal{E}_2$, it is the case that for every $w\in B_2(0,R)$; $\sum_{1\le i\le N}\ind\left\{w^T X_i\ge -1\right\}\ge \frac{N}{3}$. Let $w\in B_2(0,R)$, and $\widehat{w}\in \mathcal{N}_\epsilon$ be such that $\|w-\widehat{w}\|_2 \le \epsilon =(Cd)^{-1/2}$. Using Cauchy-Schwarz inequality, $\left|\widehat{w}^T X_i - w^T X_i\right| \le \|X_i\|_2(Cd)^{-1/2}\le 1$, where $\|X_i\|_2\le \sqrt{Cd}$ due to the event $\mathcal{E}_2$ \eqref{eq:event2}. In particular, if $\widehat{w}^T X_i\ge 0$, then $w^T X_i\ge -1$. Hence
$
\sum_{1\le i\le N}\ind\left\{w^T X_i\ge -1\right\}\ge \sum_{1\le i\le N}\ind\left\{\widehat{w}^T X_i\ge 0\right\}\ge \frac{N}{3}$.
Using now the fact $a_j\ge 0$, and $\SGM(\cdot)\ge 0$ for the sigmoid activation, we arrive at
\begin{equation}\label{eq:sigmoid-sum-lower-bd}
  \sum_{1\le j\le \overline{m}}a_j \sum_{1\le i\le N}\SGM\left(w_j^T X_i\right) \ge \frac{N}{3}\cdot\SGM(-1)\cdot\sum_{1\le j\le n}a_j.
\end{equation}
We now combine the facts $\SGM(-1) =(1+e)^{-1}$, \eqref{eq:sigmoid-important} and \eqref{eq:sigmoid-sum-lower-bd}, to obtain that on the event $\mathcal{E}_0\cap \mathcal{E}_1\cap \mathcal{E}_2$,
\[
\sum_{1\le j\le \overline{m}}a_j \le 3(1+e)(\delta+2M).
\]
Since the event $\mathcal{E}_0\cap \mathcal{E}_1\cap \mathcal{E}_2$ holds with probability at least $1-\left(3R\sqrt{Cd}\right)^d\exp\left(-\Theta(N)\right)-N\exp\left(-\Theta(d)\right)-o_N(1)$ by a union bound, the proof is complete. 
\end{proof}
\subsection{Proof of Theorem~\ref{thm:ReLU}}\label{sec:pf-ReLU}

\begin{proof}[Proof of Theorem~\ref{thm:ReLU}] Recall from~\eqref{eq:event0} 
the event $\mathcal{E}_0=\{\sum_{1\le i\le N}|Y_i|\le 2MN\}$ where $\mathbb{P}(\mathcal{E}_0)\ge $ $1-o_N(1)$. 
 
Let $(a,W)\in\mathcal{G}(\delta)$. Then, for some $\overline{m}\in\mathbb{N}$, $(a,W)\in \mathcal{G}(\overline{m},\delta)$. Using Cauchy-Schwarz inequality and the triangle inequality like in the beginning of the proof of Theorem~\ref{thm:sigmoid}; we first establish that on the event $\mathcal{E}_0$, the following holds:
\begin{equation}\label{eq:ReLU-upper-bound}
   \sum_{1\le i\le N}\sum_{1\le j\le \overline{m}}a_j \ReLU\left(w_j^T X_i\right)\le N(\delta+2M).
\end{equation}
Next, let $\mathcal{N}_\epsilon$ be a $\epsilon-$net for $B_2(0,1)$, $\epsilon>0$ to be tuned. Using Theorem~\ref{lemma:card-net}, one can ensure $\left|\mathcal{N}_\epsilon\right|\le (3/\epsilon)^d$. Fix any $\widehat{w}\in\mathcal{N}_\epsilon$. Consider the i.i.d. random variables $Y_{\widehat{w},i}\triangleq \ReLU\left(\widehat{w}^T X_i\right)$, $i\in[N]$. 
The condition (b) on the distribution of $Y_{\widehat{w},i}$ ensures that a large deviations bound is applicable. This, together with the condition (a) yield
$
\mathbb{P}\left(\sum_{1\le i\le N}Y_{\widehat{w},i} \ge \frac12\boldsymbol{\mu^*}N\right)\ge 1-\exp\left(-\Theta(N)\right)$.
Due to the distributional assumption, %, the sequence $Y_{\widehat{w}} = \ReLU\left(\widehat{w}^T X\right)$ is equidistributed over $\widehat{w}\in\mathcal{N}_\epsilon$. Hence, t
the lower bound is uniform in $\widehat{w}\in\mathcal{N}_\epsilon$. 

Taking now a union bound over $\widehat{w}\in \mathcal{N}_\epsilon$, we obtain
$ \mathbb{P}(\mathcal{E}_1)\ge 1-(3/\epsilon)^d\exp\left(-\Theta(N)\right)$ where $\mathcal{E}_1\triangleq \bigcap_{\widehat{w}\in\mathcal{N}_\epsilon}\left\{\sum_{1\le i\le N}\ReLU\left(\widehat{w}^T X_i\right)\ge \frac12\boldsymbol{\mu^*}N\right\}$.
\iffalse
\begin{multline}\label{eq:event1-ReLU}
  \mathbb{P}(\mathcal{E}_1)\ge 1-(3/\epsilon)^d\exp\left(-\Theta(N)\right),\quad\text{where}\\\mathcal{E}_1\triangleq \bigcap_{\widehat{w}\in\mathcal{N}_\epsilon}\left\{\sum_{1\le i\le N}\ReLU\left(\widehat{w}^T X_i\right)\ge \frac12\boldsymbol{\mu^*}N\right\}.
\end{multline}
\fi 
Another union bound over data $X_i$, $1\le i\le N$, yields that $\mathbb{P}(\mathcal{E}_2)\ge 1-N\exp\left(-\Theta(d)\right)$ where $\mathcal{E}_2\triangleq \left\{\|X_i\|_2^2\le Cd,1\le i\le N\right\}$.
\iffalse
\begin{multline}
\label{eq:event2-ReLU}
\mathbb{P}(\mathcal{E}_2)\ge 1-N\exp\left(-\Theta(d)\right),\quad \text{ where }\\\mathcal{E}_2\triangleq \left\{\|X_i\|_2^2\le Cd,1\le i\le N\right\}.
\end{multline}
\fi 

Choose $\epsilon\triangleq \frac{\boldsymbol{\mu^*}}{4\sqrt{Cd}}$, and assume in the remainder that we are on the event $\mathcal{E}_1\cap \mathcal{E}_2$. Next, observe that \texttt{ReLU} is $1-$Lipschitz: 
$
\left|\ReLU(x)-\ReLU(y)\right| = \left|\frac{x+|x|}{2}-\frac{y+|y|}{2}\right|\le  |x-y|$, 
using triangle inequality twice. Now, fix \emph{any} $w\in B_2(0,1)$. Let $\widehat{w}\in\mathcal{N}_\epsilon$ be the member of the net closest to $w$. Using the Lipschitz property, and the Cauchy-Schwarz, we obtain
$
\left|\ReLU\left(w^T X_i\right)-\ReLU\left(\widehat{w}^T X_i\right)\right|\le \left|w^TX_i - \widehat{w}^T X_i\right|\le \left\|w-\widehat{w}\right\|_2 \cdot \|X_i\|_2 \le \frac{\boldsymbol{\mu^*}}{4}$.
Consequently,
$
\ReLU\left(w^T X_i\right)\ge \ReLU\left(\widehat{w}^T X_i\right) -  \frac{\boldsymbol{\mu^*}}{4}$.
Summing this over $1\le i\le N$, we have
$\sum_{1\le i\le N}\ReLU\left(w^T X_i\right)\ge \sum_{1\le i\le N}\ReLU\left(\widehat{w}^T X_i\right) -\frac{\boldsymbol{\mu^*}}{4}N \ge \frac{\boldsymbol{\mu^*}}{4}N$.
Using $a_j\ge 0$, we obtain by taking $w_j$ in place of $w$:
\begin{equation}\label{eq:ReLU-lower}
\sum_{1\le j\le \overline{m}}a_j \sum_{1\le i\le N}\ReLU\left(w_j^T X_i\right)\ge  \frac{\boldsymbol{\mu^*}}{4}N\sum_{1\le j\le \overline{m}}a_j .
\end{equation}
Combining \eqref{eq:ReLU-upper-bound} and \eqref{eq:ReLU-lower}, we obtain that on the event $\mathcal{E}_0\cap \mathcal{E}_1\cap\mathcal{E}_2$,
\[
\sum_{1\le j\le \overline{m}}a_j \le 4(\delta+2M)\left(\boldsymbol{\mu^*}\right)^{-1}.
\]
Since the event $\mathcal{E}_0\cap \mathcal{E}_1\cap \mathcal{E}_2$ holds with probability at least $1-\left(12\sqrt{Cd}(\boldsymbol{\mu^*})^{-1}\right)^d\exp\left(-\Theta(N)\right)-N\exp\left(-\Theta(d)\right)-o_N(1)$ via a union bound, we complete the proof. 
\end{proof}
\subsection{Proof of Theorem~\ref{thm:Step}}\label{sec:pf-Step}

\begin{proof}[Proof of Theorem~\ref{thm:Step}]
The proof is  quite similar to that of the proof of Theorems~\ref{thm:sigmoid}/~\ref{thm:ReLU}, and is provided for completeness.

Again, recall from~\eqref{eq:event0} 
the event $\mathcal{E}_0=\{\sum_{1\le i\le N}|Y_i|\le 2MN\}$ where $\mathbb{P}(\mathcal{E}_0)\ge $ $1-o_N(1)$. %Let $(a,W)\in\mathcal{G}(\delta)$. Then, for some $\overline{m}\in\mathbb{N}$, $(a,W)\in \mathcal{G}(\overline{m},\delta)$. Using Cauchy-Schwarz inequality and the triangle inequality like in the beginning of the proof of Theorem~\ref{thm:sigmoid}; we first establish that on the event $\mathcal{E}_0$, the following holds:
%\begin{equation}\label{eq:ReLU-upper-bound}
 %  \sum_{1\le i\le N}\sum_{1\le j\le \overline{m}}a_j \ReLU\left(w_j^T X_i\right)\le N(\delta+2M).
%\end{equation}
Then, take an $(a,W)\in \mathcal{H}(\delta)$. There exists an $\overline{m}\in\mathbb{N}$ such that $(a,W)\in \mathcal{H}(\overline{m},\delta)$. Using again Cauchy-Schwarz inequality and the triangle inequality like in the beginning of the proof of Theorems~\ref{thm:sigmoid}/~\ref{thm:ReLU}; we have that on the event $\mathcal{E}_0$, the following holds:
\[
\sum_{1\le i\le N}\left|Y_i-\sum_{1\le j\le \overline{m}}a_j\Step\left(w_j^T X_i\right)\right|\le N\delta.
\]
This, together with (a) the fact that the labels are bounded, $|Y_i|\le M$; and (b) the triangle inequality; then yields
\begin{equation}\label{Eq:step-upper-bound}
    \sum_{1\le i\le N}\sum_{1\le j\le \overline{m}}a_j\Step\left(w_j^T X_i\right)\le N\left(\delta+M\right).
\end{equation}
Let $\mathcal{N}_\epsilon$ be an $\epsilon-$net for $B_2(0,1)$, where $\epsilon>0$ to be tuned appropriately. Using Theorem~\ref{lemma:card-net}, one can ensure $|\mathcal{N}_\epsilon|\le (3/\epsilon)^d$. 

Next, fix any $\widehat{w}\in \mathcal{N}_\epsilon$; and set $Z_i\triangleq \ind\left\{\widehat{w}^T X_i\ge \eta\right\}$, $1\le i\le N$ (where we drop the dependence of $Z_i$ on $\widehat{w}$ for convenience). Evidently, $Z_i$ is an i.i.d. collection of Bernoulli random variables, with $\mathbb{E}[Z_i]\ge \eta$ (due to the assumption on the distribution of $X$). Hence, using standard concentration results, $\mathbb{P}\left(\sum_{1\le i\le N}Z_i \ge N\eta/2\right)\ge 1-\exp\left(-\Theta(N)\right)$. Moreover, the lower bound is, again, uniform in $\widehat{w}$ via an exact same stochastic domination argument, like in the proof of Theorem~\ref{thm:sigmoid}.

Taking now a union bound over the net $\mathcal{N}_\epsilon$, 
\begin{equation}
\label{eq:event-step}
    \mathbb{P}(\mathcal{E}_1)\ge 1-(3/\epsilon)^d\exp\left(-\Theta(N)\right),\quad\text{where}\quad \mathcal{E}_1\triangleq \bigcap_{\widehat{w}\in\mathcal{N}_\epsilon}\left\{\sum_{1\le i\le N}\ind\left\{\widehat{w}^T X_i\ge \eta\right\}\ge N\eta/2\right\}.
\end{equation}
Furthermore, another union bound over data, $1\le i\le N$, yields
\begin{equation}
\label{eq:event2-step}
\mathbb{P}(\mathcal{E}_2)\ge 1-N\exp\left(-\Theta(d)\right),\quad \text{ where }\quad\mathcal{E}_2\triangleq \left\{\|X_i\|_2^2\le Cd,1\le i\le N\right\}.
\end{equation}
We now choose $\epsilon = \frac{\eta}{2\sqrt{Cd}}$; and assume in the remainder that we are on the event $\mathcal{E}_1\cap\mathcal{E}_2$. 

Fix any $w\in B_2(0,1)$; and let $\widehat{w}\in \mathcal{N}_\epsilon$ be such that $\|w-\widehat{w}\|_2 \le \frac{\eta}{2\sqrt{Cd}}$. Using Cauchy-Schwarz inequality, $
\left|\widehat{w}^T X_i - w^T X_i\right|\le \|w-\widehat{w}\|_2 \|X_i\|_2 \le \eta/2$, for every $i\in [N]$, since the event we are on is a subset of $\mathcal{E}_2$ in \eqref{eq:event2-step}. Observe now that $\{\widehat{w}^T X\ge \eta\}\subseteq \{w^T X\ge \eta/2\}$. Thus, on the event $\mathcal{E}_1\cap \mathcal{E}_2$, it holds that
\[
\sum_{1\le i\le N}\ind\{w^T X_i\ge \eta/2\} \ge 
\sum_{1\le i\le N}\ind\{\widehat{w}^T X_i\ge \eta/2\}\ge N\eta/2.
\]
Since $w\in B_2(0,1)$ is arbitrary, and $\Step(w^T X_i)=1$ if $w^T X_i\ge \eta/2>0$, we arrive at
\begin{equation}\label{eq:step-lower-bd}
    \sum_{1\le j\le \overline{m}} a_j\sum_{1\le i\le N}\Step\left(w_j^TX_i\right)\ge \frac{N\eta}{2}\sum_{1\le j\le \overline{m}}a_j. 
\end{equation}
We now combine \eqref{Eq:step-upper-bound} and \eqref{eq:step-lower-bd} to arrive at the conclusion that on the event $\mathcal{E}_1\cap \mathcal{E}_2$, it holds
\[
\sum_{1\le j\le\overline{m}}a_j \le 2(\delta+M)\eta^{-1}.
\]
Finally, we combine \eqref{eq:event-step} (with $\epsilon = \frac{\eta}{2\sqrt{Cd}}$) and \eqref{eq:event2-step}  via a union bound; and arrive at the conclusion that $\mathbb{P}(\mathcal{E}_1\cap \mathcal{E}_2)\ge 1-\left(6\sqrt{Cd}\cdot (\eta)^{-1}\right)^d\exp\left(-\Theta(N)\right)-N\exp\left(-\Theta(d)\right)-o_N(1)$. This concludes the proof. 

%since $\mathcal{E}_1\cap\mathcal{E}_2$ holds with probability at least \eqref{eq:event-step} 
%We only provide a sketch as it is quite similar to the proof of Theorem~\ref{thm:sigmoid}. Take instead an $\epsilon-$net $\mathcal{N}_\epsilon$ with $\epsilon = \frac{\eta}{2\sqrt{Cd}}$, take $\widehat{w}\in \mathcal{N}_\epsilon$; and consider instead the collection $\ind\{\widehat{w}^T X_i\ge \eta\}$. It follows $\mathbb{P}(\sum_{1\le i\le N}\ind\{\widehat{w}^T X_i\ge \eta\}\ge N\eta/2)\ge 1-\exp\left(-\Theta(N)\right)$, and therefore on the high probability event $\mathcal{E}_1\cap \mathcal{E}_2$ (as in the proof of Theorem~\ref{thm:sigmoid}, with appropriate modifications), $\sum_{1\le i\le N}\ind\{w^T X_i\ge \eta/2\}\ge N\eta/2$ with high probability (for every $w$). This, together with an upper bound similar to~\eqref{eq:sigmoid-important} yields the outer norm bound.
\end{proof}
\subsection{Proof of Theorem~\ref{thm:main-generalization}}\label{sec:pf-main-genel}
In this section, we establish Theorem~\ref{thm:main-generalization}. We build upon earlier results by Bartlett~\cite{bartlett1998sample} and Bartlett, Long, and Williamson~\cite{bartlett1996fat}. For the results we cite from the latter, the numbers recorded below are from the version accessed at~\url{http://phillong.info/publications/fatshat.pdf}\footnote{See the archived version at~\url{http://web.archive.org/web/20200921180645/http://phillong.info/publications/fatshat.pdf} if the link above is expired.}.

\paragraph{ The FSD of the Networks with a Bounded Outer Norm.} We now recall the definition of the fat-shattering dimension (FSD), verbatim from~\cite{bartlett1998sample}, for convenience.
\begin{definition}\label{def:fsd}
Let $X$ be an input space, $H$ be a class of real-valued functions defined on $X$ (that is, $H$ consists of functions $f:X\to\R$). Fix a $\gamma>0$, which is a certain \emph{scale} parameter. We say that a sequence $(x_1,x_2,\dots,x_m)$ of $m$ points from $X$ is $\gamma$-shattered by $H$ if there is an $r=(r_1,\dots,r_m)\in \R^m$ such that, for all $b=(b_1,\dots,b_m)\in\{-1,1\}^m$ there is an $h\in H$ satisfying $(h(x_i)-r_i)b_i\ge \gamma$. Define the fat-shattering dimension of $H$ as the function
\begin{equation}\label{eq:fsd-from-bartlett}
{\rm FSD}_H(\gamma) \triangleq \max\Bigl\{m:\text{$H$ $\gamma$-shatters some $x\in X^m$}\Bigr\}.
\end{equation}
\end{definition}

%In what follows, we combine the prior results on the FSD of such networks with our outer norm bounds to yield the promised generalization guarantees. 

We next record the following result.
%, and he in particular establishes the following. 
\begin{theorem}{\cite[Corollary~24]{bartlett1998sample}}\label{thm:bartlett}
Let $\mathcal{M}>0$, and $\sigma:\R \to[-\mathcal{M}/2,\mathcal{M}/2]$ be a non-decreasing function. Define a class $F$ of functions on $\R^d$ by
\[
F \triangleq \left\{X\mapsto \sigma\left(w^T X + w_0\right):w\in \R^d,w_0 \in\R\right\}
\]
and let
\[
H(A)\triangleq \left\{\sum_{1\le j\le \overline{m}} a_j f_j : \overline{m}\in\mathbb{N}, f_j\in F,\|a\|_1\le A\right\}
\]
where $A\ge 1$. Then for every $\gamma \le \mathcal{M}A$, 
\[
{\rm FSD}_{H(A)}(\gamma) \le \frac{c \mathcal{M}^2 A^2 d}{\gamma^2}\ln\left(\frac{\mathcal{M}A}{\gamma}\right)
\]
for some universal constant $c>0$.
\end{theorem}
Here $a=(a_j:1\le j\le \overline{m})\in\R^{\overline{m}}$ is the vector of output weights, $\|a\|_1$ is the outer norm; and $\gamma>0$ is a certain \emph{scale} parameter. Observe that $H(A)$ is precisely the class of two-layer \NN with activation function $\sigma(\cdot)$ whose outer norm is at most $A$. Per Theorem~\ref{thm:bartlett}, the FSD of the class of two-layer networks with {\bf bounded outer norm} is upper bounded by an explicit quantity. %: it has ``low complexity". 
%Bartlett then leverages the FSD bounds to devise good generalization guarantees for the architectures that he investigates. It is worth noting, however, that he establishes this link in the context of \emph{classification} setting, $Y\in\{\pm 1\}$. Since we assume $a_j\ge 0$, and the activations we study are non-negative, this does not apply to our case: the outputs of the networks we study are always non-negative. Nevertheless, we by-pass this by combining our outer norm bounds (Theorems~\ref{thm:sigmoid}-\ref{thm:Step}), Theorem~\ref{thm:bartlett} as well as building upon several other prior results tailored for the \emph{regression} setting (see below). 

\paragraph{ Some Extra Notation on Covering Numbers.} We next introduce several quantities verbatim from~\cite{bartlett1996fat}. Let $W$ be an arbitrary set, and $f:W\to\R$ be any function. For any $w=(w_1,\dots,w_N)\in W^N$, denote by $f|_w$ the $N$-tuple $\left(f(w_1),f(w_2),\dots,f(w_N)\right)\in\R^N$. For a class $\mathcal{C}$ of functions $f:W\to\R$, let $\mathcal{C}|_w\subseteq \R^N$ denotes the set 
\begin{equation}\label{eq:f-mathcal-c}
 \mathcal{C}|_w\triangleq  \Bigl\{f|_w:f\in\mathcal{C}\Bigr\} = \Bigl\{\Bigl(f(w_1),\dots,f(w_N)\Bigr):f\in\mathcal{C}\Bigr\}\subseteq \R^N.
\end{equation}

Next, recall the  covering numbers from Definition~\ref{def:eps-net}. Throughout this section, and in particular the proof of Theorem~\ref{thm:main-generalization}, we take the metric $\rho$ appearing in Definition~\ref{def:eps-net} to be the normalized $\ell_1$ distance: for any $w,\bar{w}\in\R^N$, set \[
\rho(w,\bar{w}) = \frac1N \sum_{1\le i\le N}|w_i-\bar{w}_i|. 
\]
For any $U\subseteq \R^N$, denote by $\mathcal{N}(\epsilon,U)$ the covering number of $U$ (at scale $\epsilon$) with respect to the metric $\rho$ above. That is, $\mathcal{N}(\epsilon,U)$ is the cardinality of the smallest $\mathcal{N}_\epsilon \subset U$ (if finite) such that for every $w\in U$, there exists a $\bar{w}\in N_\epsilon$ with $\rho(w,\bar{w}) = \frac1N \sum_{1\le i\le N}|w_i-\bar{w}_i|\le \epsilon$. (It is worth noting that here we flipped the order of arguments in $\mathcal{N}$ appearing in Definition~\ref{def:eps-net}. The rationale for this is to be consistent with the notation of Bartlett et al.~\cite{bartlett1996fat}.) 

Throughout this section, we often consider the following special case of $\mathcal{N}(\cdot,\cdot)$: we employ $\mathcal{N}(\cdot,\mathcal{C}|_w)$ for appropriate classes $\mathcal{C}$ of functions where $w$ is an element of the Euclidean space $\R^N$ for some $N$.

We now establish the following proposition which provides a control for the generalization gap uniformly over all two-layer \NN models with bounded outer norm. 
\begin{proposition}\label{prop:main}
Let $M,\mathcal{M},A>0$; $\sigma:\R\to[-\mathcal{M}/2,\mathcal{M}/2]$ be a non-decreasing activation function; and $\mathcal{D}$ be an arbitrary distribution on $\R^d\times \R$ for the input/label pairs $(X,Y)$ where $|Y|\le M$ almost surely. Recall the class $H(A)$ of two-layer neural networks with activation $\sigma$ and outer norm at most $A$ from Theorem~\ref{thm:bartlett}; and let $(X_i,Y_i)$, $1\le i\le N$, be i.i.d. samples drawn from $\mathcal{D}$. %Set
%\begin{equation}\label{eq:xi-in-prop}
%\xi\left(\alpha,M,\mathcal{M},A\right)\triangleq \frac{2}{\ln 2}\cdot \frac{c \cdot 128^2 \cdot \mathcal{M}^6 A^6\cdot \max\{\mathcal{M}A,2M\}^2 }{ \alpha^2}\cdot  \ln \left(\frac{128 \mathcal{M}^3A^3 \max\{\mathcal{M}A,2M\}}{\alpha}\right).
%\end{equation}
Then for any $\alpha>0$, with probability at least
\[
1-4\exp\left(\xi(\alpha,M,\mathcal{M},A)\cdot 
d \cdot \ln^2 \left(\frac{576N\mathcal{M}^2 A^2 \max\{\mathcal{M}A,2M\}}{\alpha}\right) - \frac{\alpha^2 N}{64\max\{\mathcal{M}A,2M\}^2}\right)
\]
over the draw of the training data $(X_i,Y_i)$, $1\le i\le N$, it holds that
\[
\sup_{\varphi \in H(A)} \left|\frac1N \sum_{1\le i\le N}\Bigl (\varphi(X_i)-Y_i\Bigr)^2 - \mathbb{E}_{(X,Y)\sim \mathcal{D}}\Bigl[\Bigl(\varphi(X)-Y\Bigr)^2\Bigr]\right|\le \alpha,
\]
provided
\[
N\ge 64\cdot 128 c\cdot \frac{\mathcal{M}^6 A^6 \max\{\mathcal{M}A,2M\}^2}{\alpha^2}d
\]
Here, $c,c'>0$ are absolute constants, the term $\xi$ is introduced in~\eqref{eq:xi-in-prop} and the expectation is taken with respect to a fresh sample $(X,Y)\sim \mathcal{D}$ independent of $(X_i,Y_i)$, $1\le i\le N$.
\end{proposition}
It is worth noting that while we made no attempts for simplifying the constants appearing throughout Proposition~\ref{prop:main}, we believe that they can be improved.
\begin{proof}[Proof of Proposition~\ref{prop:main}]
%We first recall certain quantities from~\cite{bartlett1996fat}. 

%For $n\in\mathbb{N}$ and $v,w\in\R^n$, define $d(v,w)\triangleq \frac1n\sum_{1\le i\le n}|v_i-w_i| = \|v-w\|_1$. 

We first provide a result established originally in~\cite[Theorem~3,p.~107]{haussler1992decision}.
\begin{theorem}\label{thm:haussler}
Let $X,Y$ be sets; $G$ be a PH-permissible class of $[0,T]$-valued functions defined on $Z \triangleq X\times Y$ where $T\in\mathbb{R}^+$, and $P$ be any distribution on $Z$. Suppose $Z_i$, $1\le i\le N$, are i.i.d. samples from $P$. Then for any $\alpha>0$, with probability at least
\[
1-4 \left(\sup_{z\in Z^{2N}} \mathcal{N}\left(\frac{\alpha}{16},G\Big|_z\right)\right)\cdot \exp\left(-\alpha^2 N/64T^2\right)
\]
over data $Z_i$, $1\le i\le N$, it holds that
\[
\sup_{g\in G}\left|\frac1N \sum_{1\le i\le N}g(Z_i) - \mathbb{E}_{Z\sim P}[g(Z)]\right|\le \alpha,
\]
where $\mathbb{E}[g(Z)]$ is taken with respect to a fresh sample (namely a sample drawn from $P$, and independent of $Z_i$). 
\end{theorem}
The version we record above is verbatim from~\cite[Theorem~13]{bartlett1996fat}. 
(The parameters $M$ and $m$ in \cite{bartlett1996fat} are replaced, respectively, with the parameters $T$ and $N$ above.)

Here, \emph{PH-permissible} refers to a rather mild measurability constraint\footnote{The letters $H$ and $P$ stand, respectively, for Haussler and Pollard---who gave a preliminary version of Theorem~\ref{thm:haussler}.}, see~\cite[Section~9.2]{haussler1992decision}. The precise details of this technicality are immaterial to us; and it is satisfied for our purposes. Moreover, $\mathcal{N}(\cdot,\cdot)$ is the covering numbers quantity defined above. %Again, we take the existing results on it as granted in the sequel.

In what follows, we take $X=\R^d$, $Y=[0,M]$ (recall that the labels are bounded almost surely by $M$) thus $Z=\R^d\times [0,M]$ and we set $P$ to  simply be $\mathcal{D}$, the distribution from which the data are drawn. We then set
\begin{equation}\label{eq:gigi}
G\triangleq \Bigl\{\left(\varphi(X)-Y\right)^2 : X\in\R^d, Y\in [0,M], \varphi(\cdot)\in H(A)\Bigr\},
\end{equation}
and take $T$ to be $\max\{\mathcal{M}A,2M\}^2$ (see below, in particular~\eqref{eq:T-chosen-this-way}).
This is nothing but the $\ell_2$ error obtained for predicting the label $Y$ with $\varphi(X)$, with $X$ being the input and $\varphi(\cdot)$ being the ``predictor". 

Upon inserting these parameters in Theorem~\ref{thm:haussler}, we obtain immediately
\begin{equation}\label{eq:sup-bd}
\sup_{\varphi \in H(A)} \left|\frac1N \sum_{1\le i\le N}\Bigl (\varphi(X_i)-Y_i\Bigr)^2 - \mathbb{E}_{(X,Y)\sim \mathcal{D}}\Bigl[\left(\varphi(X)-Y\right)^2\Bigr]\right|\le \alpha
\end{equation}
with probability at least
\begin{equation}\label{eq:haussler-prob-bd}
1-4 \left(\sup_{z\in Z^{2N}} \mathcal{N}\left(\frac{\alpha}{16},G\Big|_z\right)\right)\cdot \exp\left(-\frac{\alpha^2 N}{64\max\{\mathcal{M}A,2M\}^2}\right)
\end{equation}
over data $Z_i=(X_i,Y_i)\sim \mathcal{D}$, $1\le i\le N$. Above, we used the facts (a) $|Y|\le M$ almost surely; and (b) for any $\varphi\in H(A)$, it is the case $\varphi(X) = \sum_{1\le j\le \overline{m}} a_j \sigma\left(w_j^T X\right)$ (for an $\overline{m}\in\mathbb{N}$ and $w_j\in\R^d$, $1\le j\le \overline{m}$), where $\|a\|_1 \le A$ and $\sup_{x\in \R}|\sigma(x)|\le \mathcal{M}/2$. These together with the triangle inequality yield
\[
-\mathcal{M}A/2 \le \varphi(X) \le \mathcal{M}A/2 \quad\text{and}\quad -M\le Y\le M.
\]
Hence,
\[
-\max\{\mathcal{M}A/2,M\} \le \varphi(X),Y\le \max\{\mathcal{M}A/2,M\} \implies \Bigl(\varphi(X) - Y\Bigr)^2 \le \max\{\mathcal{M}A,2M\}^2,
\]
thus $T$ can be taken as 
\begin{equation}\label{eq:T-chosen-this-way}
T\triangleq \max\{\mathcal{M}A,2M\}^2.
\end{equation}
We next study covering number quantity
$
\sup_{z\in Z^{2N}} \mathcal{N}\left(\alpha/16,G|_z\right)
$ appearing in~\eqref{eq:haussler-prob-bd}. For this, we rely on the following result taken verbatim from~\cite[Lemma~17]{bartlett1996fat}.
\begin{lemma}\label{lemma:bart1}
Let $X$ be a set, and $F$ be a set of functions from $X$ to $[0,1]$. Then for any $\epsilon>0$ and any $N\in\mathbb{N}$, if $a\le 0$ and $b\ge 1$, we have
\[
\sup_{z\in \left(X\times [a,b]\right)^N}\mathcal{N}\left(\epsilon,\left(\ell_F\right)\Big|_z\right) \le \sup_{x\in X^N} \mathcal{N}\left(\frac{\epsilon}{3|b-a|},F\Big|_x\right).
\]
\end{lemma}
Here, $\ell_f(x,y)=(f(x)-y)^2$, $\ell_F = \{\ell_f:f\in F\}$, and for $z=(z_1,\dots,z_N)$ (where $z_i=(x_i,y_i)$), 
\[
(\ell_F)|_z =\Bigl\{ \Bigl(\ell_f(x_i,y_i)^2: 1\le i\le N\Bigr):f\in F\Bigr \},
\]
which is the notation introduced in~\eqref{eq:f-mathcal-c} with $\mathcal{C}:=\ell_F$ and $w:=z$. 
 
We take $F=H(A)$ and $\ell_F = G$ to arrive at
\begin{equation}\label{eq:cov-bd-2bartlett}
\sup_{z\in Z^{2N}}\mathcal{N}\left(\frac{\alpha}{16},G\Big|_z\right) \le \sup_{x\in \left(\R^d\right)^{2N}}\mathcal{N}\left(\frac{\alpha}{32\mathcal{M}A \max\{\mathcal{M}A,2M\}},H(A)|_x\right).
\end{equation}
Here, in addition to inserting $\alpha/16$, we also rescaled $\epsilon$ so as to reflect the fact that the functions in $H(A)$ take values in $[0,\mathcal{M}A]$. (While all the bounds established by Bartlett et al. in~\cite{bartlett1996fat} assume the output space to be $[0,1]$, they extend in a straightforward manner to any output spaces of form $[L,U]$ by rescaling corresponding parameters. This is already noted in the beginning of~\cite[Section~6]{bartlett1996fat}.) 

We next record yet another result by Bartlett et al.~\cite[Corollary~16]{bartlett1996fat}. 
\begin{lemma}\label{lemma:barttt}
Let $F$ be a class of $[0,1]$-valued functions defined on $X$, $0<\epsilon<1/2$ and $2N\ge {\rm FSD}_F(\epsilon/4)$. Then,
\[
\sup_{x\in X^N}\mathcal{N}\left(\epsilon,F|_x\right) \le \exp\left(\frac{2}{\ln 2}{\rm FSD}_F(\epsilon/4)\ln^2 \frac{9N}{\epsilon}\right),
\]
where the quantity ${\rm FSD}_F(\cdot)$ stands for the fat-shattering dimension introduced in~\eqref{eq:fsd-from-bartlett}.
\end{lemma}
(While we again skip the proof of this lemma, it is worth noting that it is obtained by combining two earlier results by Alon et al.~\cite[Lemmas~14,15]{alon1997scale}.)

Taking now $X=\R^d$ and $F=H(A)$; rescaling $\epsilon$ to $\frac{\epsilon}{\mathcal{M}A}$; and then  plugging \[
\epsilon = \frac{\alpha}{32\mathcal{M}A \max\{\mathcal{M}A,2M\}}\] as in~\eqref{eq:cov-bd-2bartlett}, we obtain
\begin{multline}
\label{eq:last-one}
    \sup_{x\in \left(\R^d\right)^{2N}}\mathcal{N}\left(\frac{\alpha}{32\mathcal{M}A \max\{\mathcal{M}A,2M\}},H(A)|_x\right) \\
    \le \exp\left(\frac{2}{\ln 2}{\rm FSD}_{H(A)}\left(\frac{\alpha}{128 \mathcal{M}^2 A^2 \max\{\mathcal{M}A,2M\}}\right)\ln^2 \left(\frac{576N\mathcal{M}^2 A^2 \max\{\mathcal{M}A,2M\}}{\alpha}\right)\right).
\end{multline}
We finally apply Theorem~\ref{thm:bartlett} above to upper bound the FSD term appearing in~\eqref{eq:last-one}. Provided
\[
\frac{\alpha}{128\mathcal{M}^2 A^2 \max\{\mathcal{M}A,2M\}} \le \mathcal{M}A \Leftrightarrow \alpha \le 128 \mathcal{M}^3 A^3 \max\{\mathcal{M}A,2M\}
\]
it holds that
\begin{multline}\label{eq:fsd-up-bd}
    {\rm FSD}_{H(A)}\left(\frac{\alpha}{128\mathcal{M}^2 A^2 \max\{\mathcal{M}A,2M\}}\right) \\
    \le \frac{128^2 c \mathcal{M}^6 A^6 \max\{\mathcal{M}A,2M\}^2}{\alpha^2}d \cdot \ln \left(\frac{128 \mathcal{M}^3A^3 \max\{\mathcal{M}A,2M\}}{\alpha}\right)
\end{multline}
where $c>0$ is the absolute constant appearing in Theorem~\ref{thm:bartlett}.

%Finally, we set $\xi$ as in~\eqref{eq:xi-in-prop}, that is
%\begin{equation}\label{eq:xi-chosen-this-wway}
%\xi\left(\alpha,M,\mathcal{M},A\right)\triangleq \frac{2}{\ln 2}\cdot \frac{c \cdot 128^2 \cdot \mathcal{M}^6 A^6\cdot \max\{\mathcal{M}A,2M\}^2 }{ \alpha^2}\cdot  \ln \left(\frac{128 \mathcal{M}^3A^3 \max\{\mathcal{M}A,2M\}}{\alpha}\right)
%\end{equation}
Finally, combining the chain of equations~\eqref{eq:haussler-prob-bd},~\eqref{eq:cov-bd-2bartlett}, \eqref{eq:last-one}, and \eqref{eq:fsd-up-bd}, %, and~\eqref{eq:xi-chosen-this-wway}, 
 we complete the proof.
\end{proof}
We finally provide a technical lemma to be used in the proof for the ReLU case.
\begin{lemma}\label{lemma:relu}
 Suppose that the distribution of $X\in\R^d$ satisfies the assumptions of Theorem~\ref{thm:ReLU}. Then,
 \begin{equation}\label{eq:distributional-lambda}
   \lambda(d)\triangleq  \sup_{w\in\R^d:\|w\|_2 = 1/\sqrt{Cd}}\mathbb{E}\Bigl[ \left|w^T X\right|^2 \ind\left\{\|X\|_2^2 >Cd\right\}\Bigr] \le \exp\Bigl(-\Theta(d)\Bigr).
      \end{equation}
\end{lemma}
The scaling $\|w\|_2=1/\sqrt{Cd}$ is required for technical reasons for the proof of the part ${\rm (b)}$ of Theorem~\ref{thm:main-generalization}.
\begin{proof}[Proof of Lemma~\ref{lemma:relu}]
Define
\[
\bar{\lambda}(d) \triangleq \sup_{w\in\R^d:\|w\|_2 = 1}\mathbb{E}\Bigl[ \left|w^T X\right|^2 \ind\left\{\|X\|_2^2 >Cd\right\}\Bigr]. 
\]
Clearly $\bar{\lambda}(d) = Cd\lambda(d)$. Since $C=O(1)$, it suffices to prove $\bar{\lambda}(d) \le \exp\left(-\Theta(d)\right)$. 

Next, fix a $w\in\R^d$ with $\|w\|_2=1$. Observe that using the inequality $e^x\ge 1+x$, we obtain
\[
e^{rw^T X} + e^{-rw^TX} \ge r\Bigl|w^TX\Bigr|,\quad\text{for any}\quad r\ge 0.
\]
Using the chain of inequalities
\[
8\left(a^4 + b^4\right) \ge 4\left(a^2+b^2\right)^2 \ge \left(a+b\right)^4,
\]
both due to Cauchy-Schwarz, we thus obtain
\[
\frac{8}{r^4}\Bigl(e^{4rw^T X}+e^{-4rw^T X}\Bigr)\ge \Bigl|w^T X\Bigr|^4.
\]
Now, take $r=s/4$ and then take the expectation of both sides to obtain
\begin{equation}\label{eq:lemma-relu-auxil}
\frac{2048}{s^4}\Bigl(M_1(s)+M_2(s)\Bigr) \ge \mathbb{E}\Bigl[\Bigl|w^T X\Bigr|^4\Bigr],
\end{equation}
where $M_1(s)$ and $M_2(s)$ are defined in Theorem~\ref{thm:ReLU}. Thus, 
\begin{align} 
\mathbb{E}\Bigl[ \left|w^T X\right|^2 \ind\left\{\|X\|_2^2 >Cd\right\}\Bigr]^2 &\le \mathbb{E}\Bigl[\Bigl|w^T X\Bigr|^4\Bigr] \mathbb{E}\Bigl[\ind\left\{\|X\|_2^2 >Cd\right\}^2\Bigr]  \label{eq:this-is-cse}\\
&=\mathbb{E}\Bigl[\Bigl|w^T X\Bigr|^4\Bigr]\mathbb{P}\Bigl(\|X\|_2^2 >Cd\Bigr)\label{eq:this-is-def-ind} \\
&\le \frac{2048}{s^4}\cdot \Bigl(M_1(s)+M_2(s)\Bigr)\cdot \exp\Bigl(-\Theta(d)\Bigr)\label{eq:this-is-lemma-relu-auxil}\\
&\le \exp\Bigl(-\Theta(d)\Bigl)\label{this-is-it},
\end{align}
where~\eqref{eq:this-is-cse} uses Cauchy-Schwarz inequality;~\eqref{eq:this-is-def-ind} uses the fact $\mathbb{E}[\ind\{E\}^2] = \mathbb{P}(E)$ valid for any event $E$;~\eqref{eq:this-is-lemma-relu-auxil} uses~\eqref{eq:lemma-relu-auxil} and the fact $\mathbb{P}\left(\|X\|_2^2>Cd\right)\le \exp(-\Theta(d))$; and finally~\eqref{this-is-it} uses the condition (b) on the distribution of $X$ stated in Theorem~\ref{thm:ReLU}. Taking square roots and taking the supremum over all $\|w\|_2=1$, we obtain $\overline{\lambda}(d)\le \exp(-\Theta(d))$; establishing Lemma~\ref{lemma:relu}.
\end{proof}
Having established Proposition~\ref{prop:main} and Lemma~\ref{lemma:relu}, we now complete the proof of Theorem~\ref{thm:main-generalization}.
\begin{proof}[Proof of Theorem~\ref{thm:main-generalization}]
Throughout the proof, we assume that $N$ is a sufficiently large polynomial in $d$ and satisfies Assumption~\ref{sump:N-exp-d}. Moreover, since the labels are bounded, $|Y|\le M$ almost surely, the $o_N(1)$ terms in Theorems~\ref{thm:sigmoid}-\ref{thm:Step} disappear, as noted previously. 

For the case of sigmoid and step activations, $\mathcal{M}$ can be taken as $2$. Thus, for the $\xi$ term appearing in Proposition~\ref{prop:main}, we simply employ $\xi(\alpha,M,2,A)$.%, and denote it by $\xi(\alpha,M,A)$ as in~\eqref{eq:xi-main-thm}.
\paragraph{Part (a).}
Define the class
    \[
    \overline{\mathcal{S}}(\delta,R) = \left\{X\mapsto \sum_{1\le j\le \overline{m}} a_j \SGM \left(w_j^T X\right):(a,W)\in \mathcal{S}(\delta,R)\right\},
    \]
    where $\mathcal{S}(\delta,R)$ is introduced in Theorem~\ref{thm:sigmoid}. Note, by the definition of $\mathcal{S}(\delta,R)$, that
    \[
    \sup_{(a,W)\in\mathcal{S}(\delta,R)}\EmpRisk{a,W}  =  \sup_{(a,W)\in\mathcal{S}(\delta,R)} \frac1N\sum_{1\le i\le N}\left(Y_i - \sum_{1\le j\le \overline{m}}a_j
\SGM\left(w_j^T X_i\right)^2\right) \le \delta^2.
    \]
    Applying Theorem~\ref{thm:sigmoid}, we find that provided $N\ge {\rm poly}(d)$, $\overline{\mathcal{S}}(\delta,R)\subset H(A)$ with probability bounded by~\eqref{eq:sgm-pb-bd}, where $H(A)$ is the class defined in Theorem~\ref{thm:bartlett} with $\sigma(\cdot)=\SGM(\cdot)$ and $A=3(1+e)(\delta+2M)$. 
    
    Finally, we (a)  set $\mathcal{M}=2$ in Proposition~\ref{prop:main}; (b) then consider $\xi(\alpha,M,2,A)$; and (c) set $\zeta(\alpha,M,A,N)$ as in~\eqref{eq:prob-term}. Combining now Theorem~\ref{thm:sigmoid} and Proposition~\ref{prop:main} via a union bound, we establish the desired conclusion. 
    
  \paragraph{ Part (b).} As the output of the ReLU is not bounded, the situation is more involved. 
  
  First, recall from Theorem~\ref{thm:ReLU} the sets
  \[
  \mathcal{G}(\overline{m},\delta)\triangleq \Bigl\{(a,W)\in\R_{\ge 0}^{\overline{m}}\times \R^{\overline{m}\times d}:\|w_j\|_2 =1,1\le j\le\overline{m}, \EmpRisk{a,W}\le \delta^2\Bigr\} \quad\text{and}\quad \mathcal{G}(\delta) \triangleq \bigcup_{\overline{m}\in\mathbb{N}}\mathcal{G}(\overline{m},\delta).
  \]
  By Theorem~\ref{thm:ReLU}, it holds that with probability bounded by~\eqref{eq:ReLU-pb-bd}, for any $(a,W)\in \mathcal{G}(\delta)$, $\|a\|_1\le 4(\delta+2M)(\boldsymbol{\mu^*})^{-1}$. Using the homogeneity of the ReLU activation, we instead rescale $w_j$ by $1/\sqrt{Cd}$; and consider throughout the sets
  \begin{equation}\label{eq:widetilde-g-set}
       \widetilde{\mathcal{G}}(\overline{m},\delta)\triangleq \Bigl\{(a,W)\in\R_{\ge 0}^{\overline{m}}\times \R^{\overline{m}\times d}:\|w_j\|_2 =\frac{1}{\sqrt{Cd}},j\in[\overline{m}], \EmpRisk{a,W}\le \delta^2\Bigr\} \quad\text{and}\quad \widetilde{\mathcal{G}}(\delta) \triangleq \bigcup_{\overline{m}\in\mathbb{N}}\mathcal{G}(\overline{m},\delta).
  \end{equation}
  Then, with probability at least
  \begin{equation}\label{eq:prob-bd-required-for-relu}
  1-\left(\frac{12\sqrt{Cd}}{\boldsymbol{\mu^*}}\right)^d\exp\left(-\Theta(N)\right)-N\exp\left(-\Theta(d)\right),
  \end{equation}
  it holds that
  \begin{equation}\label{eq:g-tilde-up-bd}
      \sup_{(a,W)\in \widetilde{G}(\delta)}\|a\|_1 \le \frac{4\sqrt{Cd}(\delta+2M)}{\boldsymbol{\mu^*}}.
  \end{equation}
 We now define an activation function, which is a ``saturated" version of the ReLU:
 \begin{equation}\label{eq:sat-rel}
     \SReLU(x) \triangleq  \begin{cases}
           0 & x<0 \\
           x &0\le x<1 \\
           1 & x\ge 1
        \end{cases}.
 \end{equation}
  Next, using a union bound over data $(X_i,Y_i)$, $1\le i\le N$, 
  \[
  \mathbb{P}\Bigl(\|X_i\|_2^2 \le Cd,1\le i\le N\Bigr)\ge 1-N\exp\left(-\Theta(d)\right).
  \]
  Hence by Cauchy-Schwarz inequality, 
  \[
 \mathbb{P}\left(\sup_{\|w\|_2 =\frac{1}{\sqrt{Cd}} } \left|w^T X_i\right|\le 1, 1\le i\le N \right)\ge 1-N\exp\left(-\Theta(d)\right).
  \]
  Consequently,  w.p. at least $1-N\exp(-\Theta(d))$ over $(X_i,Y_i)$; it holds that for all $(a,W) \in \widetilde{G}(\delta)$
  \begin{equation}\label{eq:training-of-saturated-rel}
  \frac1N\sum_{1\le i\le N}\left(Y_i - \sum_{1\le j\le \overline{m}} a_j \ReLU\left(w_j^T X_i\right)\right)^2  = 
  \frac1N\sum_{1\le i\le N}\left(Y_i - \sum_{1\le j\le \overline{m}} a_j \SReLU\left(w_j^T X_i\right)\right)^2 \le \delta^2.
   \end{equation}
Define next the class
\begin{equation}\label{class-g-bar}
    \overline{\mathcal{G}}(\delta) \triangleq \left\{X\mapsto \sum_{1\le j\le \overline{m}}a_j \SReLU\left(w_j^T X\right):(a,W)\in \widetilde{\mathcal{G}}(\delta)\right\}.
\end{equation}
 Note that, this set consists of all two-layer neural networks with (a) activation $\SReLU(\cdot)$, the saturated version of $\ReLU(\cdot)$; and (b) weights trained on the $\ReLU(\cdot)$ network.

 By Theorem~\ref{thm:ReLU} and~\eqref{eq:g-tilde-up-bd}, we find that provided $N\ge {\rm poly}(d)$, $\overline{\mathcal{G}}(\delta)\subset H(A)$ with probability given by~\eqref{eq:prob-bd-required-for-relu}, where $H(A)$ is the class defined in Theorem~\ref{thm:bartlett} with $\sigma(\cdot)=\SReLU(\cdot)$ and $A=4\sqrt{Cd}(\delta+2M)/\boldsymbol{\mu^*}$.

Observe that $\SReLU$ is a non-decreasing activation with bounded range. Hence, Proposition~\ref{prop:main} applies: one can simply take $\mathcal{M}=2$. We now apply Proposition~\ref{prop:main} with $\mathcal{M}=2$, and $A=4\sqrt{Cd}(\delta+2M)(\boldsymbol{\mu^*})^{-1}$ as in~\eqref{eq:g-tilde-up-bd}. Combining the probability bound~\eqref{eq:prob-bd-required-for-relu} and the one in Proposition~\ref{prop:main} by a union bound, we find that for every $\alpha>0$, with probability at least
\begin{equation}\label{eq:relu-prob-after-u-bd}
        1-\zeta\left(\alpha,M,\frac{4\sqrt{Cd}(\delta+2M)}{\boldsymbol{\mu^*}},N\right) - \left(\frac{12\sqrt{Cd}}{\boldsymbol{\mu^*}}\right)^d\exp\left(-\Theta(N)\right)-N\exp\left(-\Theta(d)\right)
\end{equation}
(where $\xi$ is introduced in~\eqref{eq:prob-term}) over training data $(X_i,Y_i)$, $1\le i\le N$,  it holds that
\[
\sup_{\varphi \in \overline{\mathcal{G}}(\delta)}\left|\frac1N \sum_{1\le i\le N}\left(Y_i-\varphi\left(X_i\right)\right)^2 - \mathbb{E}_{(X,Y)\sim \mathcal{D}}\Bigl[\Bigl(Y-\varphi(X)\Bigr)^2\Bigr]\right| \le \alpha,
\]
for the class $\overline{\mathcal{G}}(\delta)$ introduced in~\eqref{class-g-bar}. Recalling also~\eqref{eq:training-of-saturated-rel} which holds with probability $1-N\exp(-\Theta(d))$, we conclude that
\begin{equation}\label{eq:sup-gen-error-over-bar-class}
    \sup_{(a,W) \in\widetilde{\mathcal{G}}(\delta)} \mathbb{E}_{(X,Y)\sim \mathcal{D}}\left[\left(Y-\sum_{1\le j\le \overline{m}}a_j \SReLU\left(w_j^T X\right)\right)^2\right]\le \alpha+\delta^2,
\end{equation}
with probability at least
\begin{equation}\label{eq:relu-prob-after-u-bd-2}
        1-\zeta\left(\alpha,M,\frac{4\sqrt{Cd}(\delta+2M)}{\boldsymbol{\mu^*}},N\right) - \left(\frac{12\sqrt{Cd}}{\boldsymbol{\mu^*}}\right)^d\exp\left(-\Theta(N)\right)-2N\exp\left(-\Theta(d)\right).
\end{equation}
We next fix an $(a,W)\in\widetilde{\mathcal{G}}(\delta)$, and study the quantity
\begin{equation}\label{gen-gap-two-nets}
\Delta(a,W) \triangleq \left|\mathbb{E}_{(X,Y)\sim \mathcal{D}}\left[\left(Y-\sum_{1\le j\le \overline{m}}a_j \ReLU\left(w_j^T X\right)\right)^2\right] - \mathbb{E}_{(X,Y)\sim \mathcal{D}}\left[\left(Y-\sum_{1\le j\le \overline{m}}a_j \SReLU\left(w_j^T X\right)\right)^2\right] \right|.
\end{equation}
This quantity is nothing but the difference of generalization errors between two networks of same architecture, same number $\overline{m}$ of hidden units and same weights $(a,W)$; but different activations, $\ReLU(\cdot)$ and $\SReLU(\cdot)$. 

For convenience, denote
\[
\varphi_{SR}(X) \triangleq \sum_{1\le j\le\overline{m}} a_j\SReLU\left(w_j^T X\right) \quad\text{and}\quad \varphi_R(X) \triangleq \sum_{1\le j\le\overline{m}} a_j\ReLU\left(w_j^T X\right).
\]
In what follows, we employ the simple observation that since $a_j\ge 0$ and $0\le \SReLU(x)\le 1$, $0\le \varphi_{\rm SR}(X) \le \|a\|_1$.

Suppressing  the subscript $(X,Y)\sim \mathcal{D}$ from the expectations, we have
\begin{align}
    \Delta(a,W) &= \Bigl|\mathbb{E}\left[\left(Y-\varphi_{SR}(X)\right)^2\right] - \mathbb{E}\left[\left(Y-\varphi_{R}(X)\right)^2\right]\Bigr| \label{eq:by-def}\\ 
    & = \Bigl|\mathbb{E}\left[2Y\varphi_R(X) -2Y\varphi_{SR}(X)\right] + \mathbb{E}\left[\varphi_{SR}(X)^2 - \varphi_R(X)^2\right]\Bigr|\label{eq:simple-algebra} \\
    & \le \Bigl|\mathbb{E}\left[2Y\varphi_R(X) -2Y\varphi_{SR}(X)\right]\Bigr| + \Bigl|\mathbb{E}\left[\varphi_{SR}(X)^2 - \varphi_R(X)^2\right]\Bigr|\label{eq:triangle-ineq} \\
    &\le \mathbb{E}\Bigl[\Bigl|2Y\varphi_R(X) -2Y\varphi_{SR}(X)\Bigr|\Bigr] + \mathbb{E}\Bigl[\Bigl|\varphi_{SR}(X)^2 - \varphi_R(X)^2\Bigr|\Bigr]\label{eq:jensen}.
\end{align}
Above,~\eqref{eq:by-def} follows by the definition of $\Delta(a,W)$ per~\eqref{gen-gap-two-nets};~\eqref{eq:simple-algebra} follows after simple algebra;~\eqref{eq:triangle-ineq} follows by the triangle inequality; and~\eqref{eq:jensen} follows by the Jensen's inequality.%, $|\mathbb{E}[U]|\le \mathbb{E}[|U|]$ for any random variable $U$. 

We next study two individual terms appearing in~\eqref{eq:jensen} separately, while keeping in mind that $(a,W) \in \widetilde{\mathcal{G}}(\delta)$ implies $a_j\ge 0$ for $1\le j\le \overline{m}$ and $\|w_j\|_2 = 1/\sqrt{Cd}$ for $1\le j\le \overline{m}$. We have
\iffalse
\begin{align}
    \mathbb{E}\Bigl[\Bigl|2Y\varphi_R(X) -2Y\varphi_{SR}(X)\Bigr|\Bigr]&\le 2M\mathbb{E}\Bigl[\Bigl|\varphi_R(X) -\varphi_{SR}(X)\Bigr|\Bigr] \label{eq:labels-y-bounded}\\
    &=2M\Bigl(\mathbb{E}\Bigl[\left|\varphi_R(X) -\varphi_{SR}(X)\right|\Big|\|X\|_2^2\le Cd\Bigr]\mathbb{P}\left(\|X\|_2^2\le Cd\right)\Bigr)\label{eq:cond-1}\\ &+2M\Bigl(
    \mathbb{E}\Bigl[\left|\varphi_R(X) -\varphi_{SR}(X)\right|\Big|\|X\|_2^2> Cd\Bigr]\mathbb{P}\left(\|X\|_2^2> Cd\right)\Bigr)\label{eq:cond-2} \\
    &\le 2Me^{-\Theta(d)}\mathbb{E}\Bigl[\left|\varphi_R(X) -\varphi_{SR}(X)\right|\Big|\|X\|_2^2> Cd\Bigr]\label{eq:they-r-equal} \\
    &\le  2Me^{-\Theta(d)}\Bigl(\mathbb{E}\Bigl[\varphi_R(X)\Big|\|X\|_2^2> Cd\Bigr]+\mathbb{E}\Bigl[ \varphi_{SR}(X)\Big|\|X\|_2^2> Cd\Bigr]\Bigr)\label{eq:triangleq-again} \\
    &\le 2Me^{-\Theta(d)}\|a\|_1 \left(\sqrt{\lambda(d)}+1\right) \label{first-term-last-line}.
\end{align}
\fi
\begin{align}
    \mathbb{E}\Bigl[\Bigl|2Y\varphi_R(X) -2Y\varphi_{SR}(X)\Bigr|\Bigr]&\le 2M\mathbb{E}\Bigl[\Bigl|\varphi_R(X) -\varphi_{SR}(X)\Bigr|\Bigr] \label{eq:labels-y-bounded}\\
    &=2M\Bigl(\mathbb{E}\Bigl[\left|\varphi_R(X) -\varphi_{SR}(X)\right|\Big|\|X\|_2^2\le Cd\Bigr]\mathbb{P}\left(\|X\|_2^2\le Cd\right)\Bigr)\label{eq:cond-1}\\ &+2M\Bigl(
    \mathbb{E}\Bigl[\Bigl|\varphi_R(X) -\varphi_{SR}(X)\Bigr|\ind\Bigl\{\|X\|_2^2> Cd\Bigr\}\Bigr]\label{eq:cond-2} \\
    &\le 2M\mathbb{E}\Bigl[\Bigl|\varphi_R(X) -\varphi_{SR}(X)\Bigr|\ind\Bigl\{\|X\|_2^2> Cd\Bigr\}\Bigr]\label{eq:they-r-equal} \\
    &\le  2M\Bigl(\mathbb{E}\Bigl[\varphi_R(X)\ind\Bigl\{\|X\|_2^2> Cd\Bigr\}\Bigr]+\mathbb{E}\Bigl[ \varphi_{SR}(X)\ind\Bigl\{\|X\|_2^2> Cd\Bigr\}\Bigr]\Bigr)\label{eq:triangleq-again} \\
    &\le 2Me^{-\Theta(d)}\|a\|_1 \left(\sqrt{\lambda(d)}+1\right) \label{first-term-last-line}.
\end{align}
Here,~\eqref{eq:labels-y-bounded} uses the fact $|Y|\le M$ almost surely;~\eqref{eq:cond-2} is by the law of total expectation;~\eqref{eq:they-r-equal} uses the fact that on the event $\|X\|_2^2 \le Cd$, $\varphi_R(X) = \varphi_{SR}(X)$ since $\|w_j\|_2 =1/\sqrt{Cd}$;~\eqref{eq:triangleq-again} uses the triangle inequality; and finally~\eqref{first-term-last-line} uses the facts $0\le \SReLU(x)\le 1$ for every $x$, $a_j\ge 0$ for every $1\le j\le \overline{m}$; $\ReLU(x)\le |x|$; and
\[
\mathbb{E}\Bigl[\Bigl|w_j^T X\Bigr|\ind\{\|X\|_2^2>Cd\} \Bigr]\le \sqrt{\mathbb{E}\Bigl[\Bigl|w_j^T X\Bigl|^2\ind\Bigl\{\|X\|_2^2>Cd\Bigr\}\Bigl] \cdot \mathbb{E}\Bigl[\ind\Bigl\{\|X\|_2^2>Cd\Bigr\}\Bigr]}\le e^{-\Theta(d)} \sqrt{\lambda(d)}
\]
using Lemma~\ref{lemma:relu} and Cauchy-Schwarz inequality. Here, $\lambda(d)$ is the function defined in~\eqref{eq:distributional-lambda}.
%per~\eqref{eq:distributional-lambda} and Jensen's inequality.

We now study the second term in~\eqref{eq:jensen}. Observe that
\iffalse
\begin{align}
    \mathbb{E}\Bigl[\Bigl|\varphi_{SR}(X)^2 - \varphi_R(X)^2\Bigr|\Bigr] & =  \mathbb{E}\Bigl[\left|\varphi_{SR}(X)^2 - \varphi_R(X)^2\right|\Big|\|X\|_2^2 \le Cd\Bigr] \mathbb{P}\left(\|X\|_2^2 \le Cd\right) \\
    &+\mathbb{E}\Bigl[\left|\varphi_{SR}(X)^2 - \varphi_R(X)^2\right|\Big|\|X\|_2^2 > Cd\Bigr] \mathbb{P}\left(\|X\|_2^2 > Cd\right)\label{eq:condition-again} \\
    & = e^{-\Theta(d)}\mathbb{E}\Bigl[\left|\varphi_{SR}(X)^2 - \varphi_R(X)^2\right|\Big|\|X\|_2^2 > Cd\Bigr] \label{eq:again-omit-first} \\
    &\le e^{-\Theta(d)}\Bigr(\mathbb{E}\Bigl[\varphi_{SR}(X)^2 \Big|\|X\|_2^2 > Cd\Bigr] + \mathbb{E}\Bigl[ \varphi_R(X)^2\Big|\|X\|_2^2 > Cd\Bigr]\Bigl) \label{eq:triangleqqq-againnn}\\
    &\le e^{-\Theta(d)}\Bigl(\|a\|_1^2 +\sum_{1\le j\le\overline{m}}a_j^2 \mathbb{E}\Bigl[\ReLU\left(w_j^T X\right)^2 \Big|\|X\|_2^2 >Cd\Bigr] \Bigr. 
    \\ \Bigl. &+ 2\sum_{1\le j_1<j_2\le\overline{m}}a_{j_1}a_{j_2}  \mathbb{E}\Bigl[\ReLU\left(w_{j_1}^T X\right)\ReLU\left(w_{j_2}^T X\right) \Big|\|X\|_2^2 >Cd\Bigr]\Bigr)\label{open-up}\\
    &\le e^{-\Theta(d)}\Bigl(\|a\|_1^2 +\lambda(d)\sum_{1\le j\le \overline{m}}a_j^2 +2\lambda(d)\sum_{1\le j_1<j_2\le \overline{m}} a_{j_1}a_{j_2} \Bigr) \label{eq:jensens-again}\\
    &=e^{-\Theta(d)}\|a\|_1^2 \Bigl(\lambda(d)+1\Bigr)\label{just-a-square}.
\end{align} 
\fi

\begin{align}
    \mathbb{E}\Bigl[\Bigl|\varphi_{SR}(X)^2 - \varphi_R(X)^2\Bigr|\Bigr] & =  \mathbb{E}\Bigl[\left|\varphi_{SR}(X)^2 - \varphi_R(X)^2\right|\Big|\|X\|_2^2 \le Cd\Bigr] \mathbb{P}\left(\|X\|_2^2 \le Cd\right) \\
    &+\mathbb{E}\Bigl[\Bigl|\varphi_{SR}(X)^2 - \varphi_R(X)^2\Bigr|\ind\Bigl\{\|X\|_2^2 > Cd\Bigr\}\Bigr]% \mathbb{P}\left(\|X\|_2^2 > Cd\right)
    \label{eq:condition-again} \\
    & = \mathbb{E}\Bigl[\Bigl|\varphi_{SR}(X)^2 - \varphi_R(X)^2\Bigr|\ind\Bigl\{\|X\|_2^2 > Cd\Bigr\}\Bigr] \label{eq:again-omit-first} \\
    &\le \Bigr(\mathbb{E}\Bigl[\varphi_{SR}(X)^2 \ind\Bigl\{\|X\|_2^2 > Cd\Bigr\}\Bigr] + \mathbb{E}\Bigl[ \varphi_R(X)^2\ind\Bigl\{\|X\|_2^2 > Cd\Bigr\}\Bigr]\Bigl) \label{eq:triangleqqq-againnn}\\
    &\le \Bigl(e^{-\Theta(d)}\cdot \|a\|_1^2 +\sum_{1\le j\le\overline{m}}a_j^2 \mathbb{E}\Bigl[\ReLU\left(w_j^T X\right)^2 \ind\Bigl\{\|X\|_2^2 >Cd\Bigr\}\Bigr] \Bigr. 
    \\ \Bigl. &+ 2\sum_{1\le j_1<j_2\le\overline{m}}a_{j_1}a_{j_2}  \mathbb{E}\Bigl[\ReLU\left(w_{j_1}^T X\right)\ReLU\left(w_{j_2}^T X\right)\ind\Bigl\{\|X\|_2^2 >Cd\Bigr\}\Bigr]\Bigr)\label{open-up}\\
    &\le e^{-\Theta(d)}\Bigl(\|a\|_1^2 +\lambda(d)\sum_{1\le j\le \overline{m}}a_j^2 +2\lambda(d)\sum_{1\le j_1<j_2\le \overline{m}} a_{j_1}a_{j_2} \Bigr) \label{eq:jensens-again}\\
    &=e^{-\Theta(d)}\|a\|_1^2 \Bigl(\lambda(d)+1\Bigr)\label{just-a-square}.
\end{align} 
Indeed,~\eqref{eq:condition-again} is again by the law of total expectation;~\eqref{eq:again-omit-first} uses the fact that on $\|X\|_2^2 \le Cd$, $\varphi_{SR}(X)=\varphi_R(X)$ since $\|w_j\|_2 = 1/\sqrt{Cd}$;~\eqref{eq:triangleqqq-againnn} uses triangle inequality;~\eqref{open-up} is obtained by opening the parantheses while using $a_i \ge 0$, $0\le \SReLU(x)\le 1$;~\eqref{eq:jensens-again} uses the fact $a_j\ge 0$, Lemma~\ref{lemma:relu} as well as the Cauchy-Schwarz inequality
\begin{align*}
 \mathbb{E}\Bigl[\ReLU\left(w_{j_1}^T X\right)\ReLU\left(w_{j_2}^T X\right)\ind\Bigl\{\|X\|_2^2 >Cd\Bigr\}\Bigr] &\le  \sqrt{\mathbb{E}\Bigl[\ReLU^2\left(w_{j_1}^T X\right) \ind\Bigl\{\|X\|_2^2 >Cd\Bigr\}\Bigr]}\times \\
 & \sqrt{\mathbb{E}\Bigl[\ReLU^2\left(w_{j_2}^T X\right) \ind\Bigl\{\|X\|_2^2 >Cd\Bigr\}\Bigr] } \\
 &\le \lambda(d),
\end{align*}
since $\ReLU(x)\le |x|$. Finally,~\eqref{just-a-square} is obtained by just noticing that for $a_j\ge 0$,
\[
\|a\|_1^2  = \left(\sum_{1\le j\le\overline{m}}a_j\right)^2 = \sum_{1\le j\le \overline{m}}a_j^2 + 2 \sum_{1\le j_1<j_2\le \overline{m}}a_{j_1}a_{j_2}. 
\]
We now combine~\eqref{first-term-last-line} and~\eqref{just-a-square} to upper bound the right hand side of~\eqref{eq:jensen} and arrive at
\[
 \Delta(a,W) \le 2Me^{-\Theta(d)}\|a\|_1 \left(\sqrt{\lambda(d)}+1\right) + e^{-\Theta(d)}\|a\|_1^2 \Bigl(\lambda(d)+1\Bigr). 
\]
Since $\|a\|_1 \le 4\sqrt{Cd}(\delta+2M)/\boldsymbol{\mu^*}$ on $\widetilde{\mathcal{G}}(\delta)$ as recorded in~\eqref{eq:g-tilde-up-bd}, we obtain
\begin{equation}\label{eq:close-to-end}
\sup_{(a,W)\in\widetilde{\mathcal{G}}(\delta)}\Delta(a,W) \le e^{-\Theta(d)}\left(\frac{8M\sqrt{C}(\delta+2M)}{\boldsymbol{\mu^*}}\sqrt{d}\left(\sqrt{\lambda(d)}+1\right) + \frac{16C(\delta+2M)^2}{\boldsymbol{\mu^*}^2}d\Bigl(\lambda(d)+1\Bigr)\right).
\end{equation}
Recall that $\lambda(d)\le \exp(-\Theta(d))$ by~\eqref{eq:distributional-lambda}. Note that as long as $M,C,\delta,\boldsymbol{\mu^*}=\exp(o(d))$ as well, the term on the right hand side of~\eqref{eq:close-to-end} is $e^{-\Theta(d)}$. 

We finally combine~\eqref{eq:sup-gen-error-over-bar-class},~\eqref{gen-gap-two-nets}; and~\eqref{eq:close-to-end} to obtain
\[
\sup_{(a,W)\in \mathcal{G}(\delta)} \mathbb{E}_{(X,Y)\sim \mathcal{D}}\left[\left(Y-\sum_{1\le j\le \overline{m}}a_j \ReLU\left(w_j^T X\right)\right)^2\right]\le \alpha+\delta^2+e^{-\Theta(d)}
\]
with probability at least
\[
 1-\zeta\left(\alpha,M,\frac{4\sqrt{Cd}(\delta+2M)}{\boldsymbol{\mu^*}},N\right) - \left(\frac{12\sqrt{Cd}}{\boldsymbol{\mu^*}}\right)^d\exp\left(-\Theta(N)\right)-2N\exp\left(-\Theta(d)\right),
\]
as shown in~\eqref{eq:relu-prob-after-u-bd-2}. This concludes the proof of Part (b). 
%Since $\SReLU$ is nothing but a saturated version of $\ReLU$, and $\|X\|_2 \le \sqrt{Cd}$ w.h.p.

%We now relate the generalization error between two networks having the same number $\overline{m}$ of hidden units and the same weights $(a,W)$, but different activations

  \paragraph{ Part (c).}  This is quite similar to Part (a).
  
  Define the class
  \[
  \overline{\mathcal{H}}(\delta)\triangleq \left\{X\mapsto \sum_{1\le j\le \overline{m}}a_j \Step \left(w_j^T X\right):(a,W)\in\mathcal{H}(\delta)\right\}
  \]
   where $\mathcal{H}(\delta)$ is introduced in Theorem~\ref{thm:Step}. Note, by definition, that
    \[
   \sup_{(a,W)\in\mathcal{H}(\delta)} \EmpRisk{a,W}  =\sup_{(a,W)\in\mathcal{H}(\delta)} \frac1N\sum_{1\le i\le N}\left(Y_i - \sum_{1\le j\le \overline{m}}a_j
\Step\left(w_j^T X_i\right)^2\right) \le \delta^2.
    \]
     Applying Theorem~\ref{thm:Step}, we find that provided $N\ge {\rm poly}(d)$, $\overline{\mathcal{H}}(\delta)\subset H(A)$ w.h.p.
    , where $H(A)$ is the class defined in Theorem~\ref{thm:bartlett} with $\sigma(\cdot)=\Step(\cdot)$ and $A=2(\delta+2M)/\eta$. 
    
    Like in the previous case, we then (a)  set $\mathcal{M}=2$ in Proposition~\ref{prop:main}; (b) then let $\xi(\alpha,M,2)$ to be $\xi(\alpha,M,2,A)$; and (c) set $\zeta(\alpha,M,A,N)$ as in~\eqref{eq:prob-term}. Combining now Theorem~\ref{thm:Step} and Proposition~\ref{prop:main} via a union bound, we establish the desired conclusion.
    \end{proof}
\bibliographystyle{amsalpha}
\bibliography{bibliographh}

\end{document}